\documentclass{article}

\PassOptionsToPackage{sort&compress}{natbib}
\usepackage[preprint]{corl_2026} 
\newcommand{\method}{CORD-SLS}
\usepackage{booktabs}
\usepackage{multirow}
\usepackage{amsmath,amssymb}
\usepackage{makecell}
\usepackage{amsthm}

\newtheorem{theorem}{Theorem}

\usepackage{enumitem}
\usepackage{graphicx}
\usepackage{adjustbox}
\usepackage{float}
\usepackage{tikz}
\usepackage{subcaption}
\usepackage{fontawesome5}
\usepackage[most]{tcolorbox}
\newtcbox{\linkpill}{
  on line,
  box align=base,
  colback=gray!10,
  colframe=gray!10,
  arc=1.5mm,
  boxrule=0pt,
  left=2pt,
  right=2pt,
  top=1pt,
  bottom=1pt
}
\usepackage{wrapfig}
\definecolor{slsorange}{HTML}{38A4F5}

\newcommand{\vsubdividershort}{%
    \tikz[baseline=-0.5ex]\draw[dotted, black!90, line width=0.6pt] (0,0) -- (0,1.4cm);
}
\newcommand{\vsubdivider}{%
    \tikz[baseline=-0.5ex]\draw[dotted, black!90, line width=0.6pt] (0,0) -- (0,2.0cm);
}

\title{\looseness-1Robustness without Wrinkles: \hspace{-1pt}Parallel Simulation and Robust MPC for Certified Deformable Manipulation}

\author{
  Wei-Chen Li$^\star$, Jeffrey Fang$^\star$, Sasanka Polisetti, Yuexi Song, Glen Chou\\
  Georgia Institute of Technology, Atlanta, GA 30308\\
  \texttt{\{wli777, jfang301, spolisetti6, ysong644, chou\}@gatech.edu} \\
  $^\star$ Equal contribution, order selected by coin flip\\
  \linkpill{\href{https://trustworthyrobotics.github.io/CORD-SLS/}{\textcolor{slsorange}{\faGlobe\ Website}}}
\quad
\linkpill{\href{https://github.com/trustworthyrobotics/CORD-SLS}{\textcolor{slsorange}{\faGithub\ Code}}}
\quad
\linkpill{\href{https://youtu.be/glPFcKFUvWo}{\textcolor{slsorange}{\faYoutube\ Video}}}
}

\begin{document}
\maketitle

\begin{abstract}
    We present CORD-SLS, a real-time control method for safe deformable object manipulation, with a focus on ropes and cloth. At its core is a GPU-parallel differentiable simulator with contact smoothing which enables efficient gradient-based planning through intermittent contact. To robustly satisfy constraints under model and sensing uncertainty, we develop a real-time, GPU-parallel output-feedback robust model predictive control (MPC) algorithm that plans with this simulator. We further show that the simulator accelerates model-based RL for training neural manipulation policies. To improve real-world robustness, we use conformal prediction to calibrate visual-feedback and perception-error bounds for MPC, producing reachable tubes that enable high-probability safe control. We evaluate CORD-SLS on high-dimensional, contact-rich rope and cloth manipulation tasks in simulation and hardware, including obstacle avoidance, routing, folding, and smoothing. Across settings, CORD-SLS achieves millisecond-speed planning, exceeding baselines in safety, speed, and task success. 
\end{abstract}


\begin{figure}[H]
    \centering
    \includegraphics[width=\linewidth]{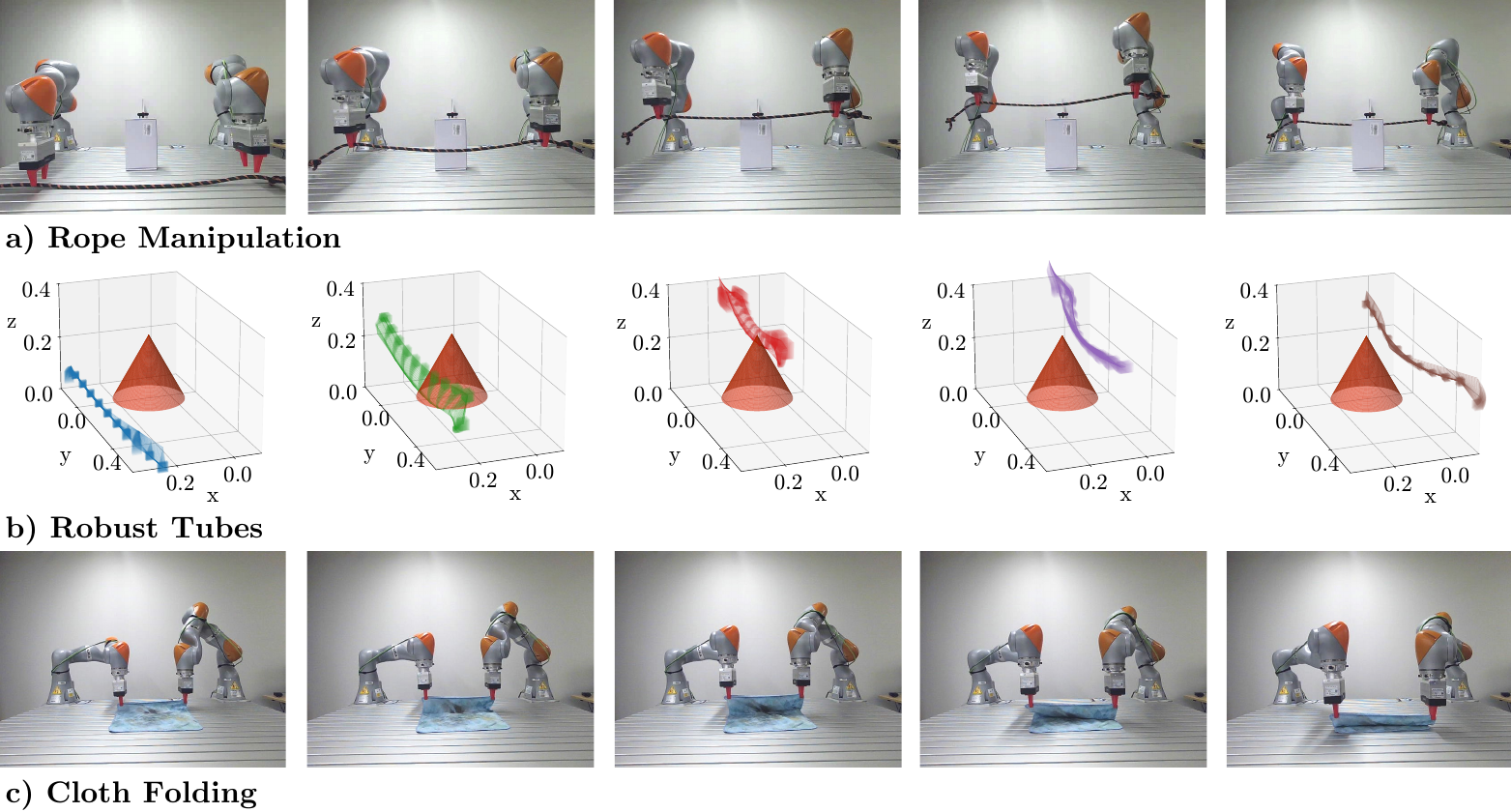}
    \caption{
    Hardware results showcase \method's real-world efficacy. In \textbf{a)}, \method~robustly solves a constrained rope manipulation task, lifting the rope and navigating over an obstacle with our contact-smoothed differentiable simulator. In \textbf{b)}, we show the associated robust tubes computed by \method~ when avoiding the obstacle (red). In \textbf{c)}, \method~extends to cloth folding.
    }
    \label{fig:hardware_experiments}
\end{figure}
\vspace{-20pt}
\section{Introduction}
\vspace{-7pt}

Manipulation of deformable objects, like ropes and cloth, is important for safety- and performance-critical applications in home robots, logistics, and surgery \cite{yin2021modeling}. Yet controlling these objects is difficult: deformable objects often lack compact state representations, and convenient approximations, such as mesh-based cloth, require planning over many degrees of freedom and high-dimensional nonlinear dynamics. Manipulation also requires deciding where and when to make contact, as well as constraint satisfaction, e.g., obstacle avoidance, goal reaching, and task requirements, despite uncertainty in dynamics and perception. Simulators can provide useful models to plan such behavior, if they are sufficiently fast, accurate, and planner-friendly.
However, prior work leaves a gap between scalable planning, deformable contact dynamics, and robustness. Classical methods plan tractably via geometric abstractions, but produce open-loop plans with limited robustness. Learned policies can plan in visual or latent spaces, but lack robustness assurances and require slow training. Differentiable simulators enable fast gradient-based planning, but discontinuous gradients arising from na\"ive contact formulations can prevent efficient contact discovery, required for manipulation. 

To close these gaps, we propose \method~(\underline{C}ontact-smoothed \underline{O}utput-feedback \underline{R}obust control of \underline{D}eformables via \underline{S}ystem \underline{L}evel \underline{S}ynthesis), a framework for real-time differentiable simulation and robust perception-based manipulation planning for deformable objects. We first develop a GPU-parallel differentiable simulator with contact smoothing, providing planning-friendly gradients at real-time simulation rates. We then plan with this simulator via a GPU-parallel robust MPC method based on system level synthesis (SLS), enabling real-time planning in high-dimensional settings. To handle perception uncertainty, we introduce a GPU-parallel output-feedback SLS formulation for robust vision-based control. We also show that the simulator can accelerate model-based reinforcement learning (MBRL) by enabling neural control policies to be trained in a sample-efficient fashion via analytical policy gradients. Finally, we calibrate data-driven worst-case uncertainty bounds for robust MPC, yielding probabilistic robustness guarantees in realistic settings.
Our contributions are:

\vspace{-6pt}
\begin{itemize}[leftmargin=1.1em,itemsep=0pt,topsep=5pt,parsep=0pt]
    \item A GPU-parallel, differentiable, real-time, contact-smoothed deformable object simulator.
    \item A GPU-parallel method for real-time synthesis of robust output-feedback control policies for deformable manipulation via reachability-constrained SLS, along with an MBRL method for training neural control policies that is accelerated via analytical policy gradients from the simulator.
    \item A real-world vision-based implementation using keypoint tracking, together with  calibration of perception and model-error bounds via conformal prediction (CP) for high-probability safety.
    \item Evaluation on contact-rich rope and cloth manipulation tasks in simulation and hardware (obstacle avoidance, folding, and flattening), outperforming baselines in safety, speed, and task success.
\end{itemize}
\vspace{-12pt}
\section{Related Work}
\vspace{-6pt}

\paragraph{Deformable Object Simulation}
Position-based dynamics (PBD) and its extension, XPBD, simulate deformable objects by projecting particles onto geometric constraints, enabling fast and stable simulations~\cite{muller2007position,ruan2018accounting, liu2022differentiable}.
However, PBD-based methods achieve robustness at the expense of physical fidelity: effective stiffness depends on the timestep and the number of solver iterations, and existing differentiable variants do not fully address rigid-deformable contact and friction.
More physically grounded models have been developed for deformable linear objects (DLOs). In particular, Cosserat rod and discrete elastic rod formulations capture bending, twisting, stretching, and shearing, and have been extended to handle frictional contact~\cite{gazzola2018forward,li2025discrete,bergou2008discrete,bergou2010discrete}. Differentiable rod simulators support system identification, but are less suited to contact-rich control because planning requires explicit contact-mode discovery~\cite{chen2024differentiable}. For cloth simulation, methods based on discrete shells, Kirchhoff-Love shells, and projective dynamics efficiently model thin surfaces and can incorporate dry frictional contact~\cite{grinspun2003discrete,wen2023kirchhoff,ly2020projective}. However, the nonconvexity of contact resolution often leads to inconsistent gradients in differentiable cloth simulators~\cite{liang2019differentiable,li2022diffcloth}. Our method uses a convex formulation of contact and smoothed dynamics to improve contact discovery during planning, producing contact-seeking gradients, while forward rollouts and execution use the nonsmooth dynamics to preserve fidelity~\cite{suh2022differentiable,suh2022bundled,li2026certified}.

\vspace{-10pt}
\paragraph{Planning}

\looseness-1Deformable-object planning is difficult because perception, dynamics, and control are tightly coupled \cite{yin2021modeling}, and models must trade off fidelity and speed \cite{jimenez2012survey, arriola2020modeling}. Unlike imitation learning \cite{seita2021learning, salhotra2022learning}, planners do not require expert data. 
Classical planners use reduced geometric abstractions for speed ~\cite{bayazit2002probabilistic,saha2006motion}, and recent hybrids combine global planning with local control or simplified dynamics~\cite{mcconachie2020manipulating,wang2025goal}. These methods improve tractability but yield open-loop plans with limited robustness.
Learned models predict perceptual or latent dynamics~\cite{shen2024action,zhang2025manipulating,yan2021learning,ma2022learning,li2024deformnet} for planning, but require extensive training and lack robustness guarantees. Higher-level abstractions such as latent roadmaps extend planning horizons and can encode contact transitions~\cite{lippi2020latent,yamada2024d,lin2022planning}, but rely on task-specific skill libraries. Differentiable simulators enable gradient-based planning ~\cite{chen2024differentiableparticles,zhang2024achieving,chen2022daxbench,yoshida2015simulation}, but discontinuous gradients impede contact discovery for planning ~\cite{chen2022daxbench,suh2022differentiable,li2026certified}. Our method instead uses contact smoothing to plan through contact without bespoke skills, while handling bounded sim-to-real model error~\cite{blanco2024benchmarking}.

\looseness-1Safety and robustness are less-studied. Some methods plan nominally-safe motions \cite{huang2023deformable, zhang2024achieving} or use barrier functions to override unsafe actions \cite{tang2024learning, aksoy2025planning}; however, these rely on approximate models without accounting for model error, provide only local corrections, and do not address contact discovery. Reachability methods evaluate DLO safety \cite{zhang2025certifiably} but do not consider contact-rich planning or cloth manipulation under perception error. 
Other methods improve robustness through model trust regions \cite{mitrano2021learning}, gradient refinement on perturbed DLO systems \cite{envall2026robust}, or MPC demonstrations \cite{wang2025robot}, but lack formal guarantees and contact-rich planning.
While safe output-feedback control has been considered in recent work \cite{yang2023safe, chou2022safe, chou2023synthesizing}, these methods struggle to scale to the high-dimensional state spaces of deformable objects.
In contrast, we unify differentiable simulation, contact smoothing, and output-feedback robust MPC via SLS, producing contact-seeking gradients while robustly accounting for model and perception uncertainty. We build on GPU-parallel reachability \cite{fang2026safe} via SLS \cite{anderson2019system,leeman2025robust}, extending it beyond full-state feedback. Unlike existing slow CPU-bound output-feedback SLS methods \cite{leeman2026vision}, our GPU-parallel solver enables real-time robust control.
\vspace{-3pt}
\section{Problem Statement}
\vspace{-3pt}

We consider a known state representation $q \in \mathcal{Q} \subseteq \mathbb{R}^{n_q}$ for a deformable object of interest (e.g., 3D positions of links for rope, or of mesh node positions for cloth) actuated by control input $u \in \mathcal{U} \subseteq \mathbb{R}^{n_u}$ and governed by uncertain dynamics \eqref{eq:dynamics_real}, where $f: \mathcal{Q}\times \mathcal{U} \rightarrow \mathcal{Q}$. To account for perception latency, we model the system as being observed with uncertain time-delayed measurements $y \in \mathcal{Y} \subseteq \mathbb{R}^{n_y}$ according to an observation map $h:\mathcal{Q}\rightarrow \mathcal{Y}$ \eqref{eq:measurements_real},

\vspace{-13pt}
\begin{nolinenumbers}
\begin{subequations}\label{eq:system_real}
\noindent
\begin{minipage}[t]{0.48\textwidth}
\vspace{0pt}
\begin{equation}
    q_{k+1} = f(q_k, u_k) + E(q_k)w_k,
    \label{eq:dynamics_real}
\end{equation}
\end{minipage}
\hfill
\begin{minipage}[t]{0.48\textwidth}
\vspace{0pt}
\begin{equation}
    y_{k+1} = h(q_k) + F(q_k)e_k,
    \label{eq:measurements_real}
\end{equation}
\end{minipage}
\end{subequations}
\end{nolinenumbers}
\looseness-1where $E: \mathcal{Q} \rightarrow \mathbb{R}^{n_q \times n_q}$ and $F: \mathcal{Q} \rightarrow \mathbb{R}^{n_y \times n_y}$ and $w \in \mathcal{B}^{n_q} := \{w \in \mathbb{R}^{n_q} \mid \Vert w\Vert_2 \le 1\}$ and $e \in \mathcal{B}^{n_y}$ are bounded disturbances.
We aim to design 1) a differentiable deformable object simulator that approximates the dynamics \eqref{eq:dynamics_real} while remaining friendly for planning and learning, and 2) an efficient and robust planning and control algorithm that uses this simulator to solve reachability-aware MPC problems for safe deformable object manipulation. We consider two specific problems:\vspace{4pt}\\
\textbf{Problem 1.} \textit{Differentiable contact-smoothed deformable simulator.}
Develop a simulator for deformable linear objects and cloth that 1) is differentiable, 2) GPU-parallel, and 3) provides informative gradients for planning and policy learning through contact, defining the planner dynamics
\vspace{-2pt}
\begin{equation}\label{eq:dynamics_sim}
    q_{k+1} = f(q_k, u_k).
\end{equation}
\vspace{-2pt}
\hspace{-6pt} \textbf{Problem 2.} \textit{Reachability-informed MPC for robust deformable manipulation.} Optimize a nominal length-$N$ state $\textbf{z}:=\{z_k\}_{k=0}^N$ and control $\textbf{v}:=\{v_k\}_{k=0}^{N-1}$ trajectory that is feasible for the simulator dynamics \eqref{eq:dynamics_sim}, together with a causal output-feedback controller $\pi := (\pi_0, \ldots, \pi_{N-1})$,  where $\pi_t: \mathcal{Y}^{t+1} \rightarrow \mathcal{U}$ stabilizes the true deformable dynamics \eqref{eq:dynamics_real} about $(\textbf{z}, \textbf{v})$ using the history of observations, and ensures closed-loop satisfaction of state and control constraints $\mathcal{S} \subseteq \mathcal{Q}\times \mathcal{U}$.
\vspace{-3pt}
\section{Method}
\vspace{-3pt}
We develop a differentiable deformable-object simulator with contact smoothing for gradient-based planning (Sec.~\ref{sec:simulator}), a GPU-parallel output-feedback SLS framework that uses it to enable real-time robust control under perception uncertainty (Sec.~\ref{sec:gpu_sls}), a sample-efficient model-based RL strategy for learning neural manipulation policies that is accelerated by analytical policy gradients from the differentiable simulator (Sec. \ref{sec:mbrl}), and a CP-based calibration procedure for translating learned perception errors into high-probability reachable-tube guarantees (Sec.~\ref{sec:conformal}).

\vspace{-6pt}
\subsection{Differentiable Deformable Object Simulator with Contact Smoothing}\label{sec:simulator}
\vspace{-4pt}
Existing deformable simulators often assume that deformable objects remain attached to the robot or rely on explicit time-stepping schemes that require small time steps for numerical stability. In contrast, we introduce a novel simulator that correctly captures switching between inactive and active constraints while leveraging an implicit time-stepping scheme, all while retaining differentiability.

Consider a nodal system with $n_v$ nodes, whose 3D coordinates, together with the $n_u$ actuated degrees of freedom (DoFs), form the configuration vector $q \in \mathbb{R}^{n_q}$, where $n_q = 3n_v + n_u$. For an elastic energy $E(q)$, the corresponding internal elastic force is given by $-\partial E(q)/\partial q$. Under a quasi-static assumption, the system satisfies
\begin{equation}
\label{eq:dynamics}
    M \frac{q^+ - q}{\delta t^2}
    + \frac{\partial E(q^+)}{\partial q^+}
    =
    \tau(u)
    + \sum_{i=1}^{n_c} J_i(q)^\top\, \lambda_i,
\end{equation}
\looseness-1where $M$ is the mass matrix, $\tau(u)$ are the external generalized forces, including gravity acting on the nodal system and actuation forces applied to the actuated DoFs, $\lambda_i \in \mathbb{R}^3$ is the $i$th constraint force, $\delta t$ is the discrete timestep, and $J_i \in \mathbb{R}^{3 \times n_q}$ is the corresponding constraint Jacobian. Both the elastic and constraint forces are evaluated at the next time step, yielding an implicit time-stepping scheme for improved numerical stability. For holonomic constraints, such as a gripper rigidly interacting with the nodal system, we suppose an active constraint $c_i(q) = 0 \in \mathbb{R}^3$ enforces the gripper attachment. The next configuration $q^+$ and constraint force $\lambda_i$ must satisfy the complementarity condition
\begin{equation} \label{eq:holonomic-constraint}
    0 = J_i(q)\, (q^+ - q) + c_i(q)
    \;\perp\;
    \lambda_i \in \mathbb{R}^3,
\end{equation}
where the constraint Jacobian is defined as $J_i(q) := \partial c_i(q) / \partial q$. 
For contact constraints, we adopt Anitescu's convex relaxation~\cite{anitescu2006optimization}. Here, the next step configuration $q^+$ and contact force $\lambda_i$ satisfy
\begin{equation} \label{eq:contact-constraint}
    \mathcal{F}_\mu^* \ni J_i(q)\, (q^+ - q) + \begin{bmatrix} \phi_i(q) & 0 & 0 \end{bmatrix}^\top
    \;\perp\;
    \lambda_i \in \mathcal{F}_\mu,
\end{equation}
where $\phi_i(q)$ is the signed distance of contact pair $i$, $\mathcal{F}_\mu$ is the friction cone, and $\mathcal{F}_\mu^*$ is its dual cone. \vspace{5pt} \\
\textbf{Simulation Solver}
We solve the dynamics in \eqref{eq:dynamics} subject to the holonomic and contact constraints in \eqref{eq:holonomic-constraint} and \eqref{eq:contact-constraint} using a two-stage procedure. The first stage computes the system response under the holonomic constraints, while the second stage resolves contact interactions. In the first stage, we solve for an intermediate configuration $q^*$ satisfying
\begin{equation} \label{eq:dynamics-first-stage}
    F(q^*; q, u) :=
    M \frac{q^* - q}{\delta t^2}
    + \frac{\partial E(q^*)}{\partial q^*}
    - \tau(u)
    + \sum_{i=1}^{n_h} J_i(q)^\top D_i\, \bigl( J_i(q)\, (q^* - q) + c_i(q) \bigr)
    = 0,
\end{equation}
where $n_h$ denotes the number of holonomic constraints,  and $D_i \in \mathbb{R}^{3 \times 3}$ is a diagonal stiffness matrix associated with the $i$th holonomic constraint. When the constraint is active, $D_i$ is chosen sufficiently large to enforce \eqref{eq:holonomic-constraint}; otherwise, $D_i = 0$. 
\eqref{eq:dynamics-first-stage} is solved using Newton's method and implicitly defines a mapping from the current configuration $q$ to the intermediate configuration $q^*$. In the second stage, we solve the optimization problem
\begin{equation} \label{eq:dynamics-second-stage}
\begin{aligned}
    \min_{q^+} \quad  & \frac{1}{2} (q^+ - q^*)^\top\, \bigl[ \underbrace{ M + \frac{\partial^2 E(q^*)}{\partial (q^*)^2} + \sum_{i=1}^{n_h} J_i(q)^\top D_i J_i(q)}_{P} \bigr] \,(q^+ - q^*) \\
    \text{s.t.} \quad & J_i(q^*)\, (q^+ - q^*) + \begin{bmatrix} \phi_i(q^*) & 0 & 0 \end{bmatrix} \in \mathcal{F}_\mu^*, \quad \forall i = n_h\!+\!1, \dots, n_c.
\end{aligned}
\end{equation}
The KKT conditions of \eqref{eq:dynamics-second-stage} recover the contact constraint in \eqref{eq:contact-constraint} together with $\textstyle P\, (q^+\!-\!q^*) \!= \!\sum_{i=n_h\!+\!1}^{n_c} J_i(q)^\top\, \lambda_i$, which corresponds to a linearization of the dynamics about the intermediate configuration $q^*$.
Given a $(q,u)$ pair, solving \eqref{eq:dynamics-first-stage}--\eqref{eq:dynamics-second-stage} gives the next configuration $q^+$.
The solvers for \eqref{eq:dynamics-first-stage}--\eqref{eq:dynamics-second-stage} are implemented in JAX, allowing GPU-accelerated matrix factorizations.

\textbf{Differentiating Through Simulation}
To obtain gradients through the dynamics, we differentiate through the optimization problems in \eqref{eq:dynamics-first-stage} and \eqref{eq:dynamics-second-stage}. Rather than differentiating through the unrolled iterative solvers, we employ implicit differentiation. For the first-stage problem \eqref{eq:dynamics-first-stage}, let $\theta := (q,u)$ denote the inputs to the dynamics. Differentiating the implicit equation defining $q^*$ yields
\begin{equation}
    \frac{\partial F}{\partial q^*} \frac{\partial q^*}{\partial \theta} + \frac{\partial F}{\partial \theta} = 0
    \implies
    \frac{\partial q^*}{\partial \theta} = - \Bigl( \frac{\partial F}{\partial q^*} \Bigr)^{-1} \frac{\partial F}{\partial \theta}.
\end{equation}
For the second-stage problem \eqref{eq:dynamics-second-stage}, we differentiate its KKT conditions, denoted by $G(q^+, \lambda; q^*) = 0$, with respect to $q^*$. This gives
\begin{equation}
    K
    \begin{bmatrix}
        \frac{\partial q^+}{\partial q^*} \\
        \frac{\partial \lambda}{\partial q^*}
    \end{bmatrix}
    + \frac{\partial G}{\partial q^*} = 0
    \implies
    \frac{\partial q^+}{\partial q^*} = - \begin{bmatrix} I & 0 \end{bmatrix} K^{-1} \frac{\partial G}{\partial q^*},
\end{equation}
where $K$ denotes the KKT matrix of the optimization problem. The simulator can thus be treated as a differentiable layer. For a scalar objective $\mathcal{L}$, the adjoint with respect to $\theta = (q,u)$ is given by
\begin{equation}  \label{eq:diff-sim}
    \frac{\partial \mathcal{L}}{\partial \theta}
    =
    \frac{\partial \mathcal{L}}{\partial q^+}
    \frac{\partial q^+}{\partial q^*}
    \frac{\partial q^*}{\partial \theta}
    =
    \begin{bmatrix}
        \dfrac{\partial \mathcal{L}}{\partial q^+} & 0
    \end{bmatrix}
    K^{-1}
    \frac{\partial G}{\partial q^*}
    \left(
        \frac{\partial F}{\partial q^*}
    \right)^{-1}
    \frac{\partial F}{\partial \theta}.
\end{equation}
This expression naturally defines a custom backward pass for the simulator layer. Starting from the adjoint of $q^+$, we first solve a linear system involving $K^\top$ and compute a vector-Jacobian product with $\partial G / \partial q^*$ to obtain the adjoint of $q^*$. We then solve a second linear system involving $(\partial F / \partial q^*)^\top$ and compute a vector-Jacobian product with $\partial F / \partial \theta$ to obtain the adjoint of $\theta$. Since the factorization of $K$ and $\partial F / \partial q^*$ are byproducts of the forward solve, the backward pass incurs no additional cost. 

\textbf{Contact Smoothing} Contact dynamics often exhibit a vanishing-gradient issue: prior to contact activation, the control inputs associated with the actuated DoFs have no influence on the state of the deformable object \cite{li2026certified}. As a result, gradient-based optimizers may stagnate due to the absence of informative gradient signals, failing to reach the goal.

To address this issue, we adopt the idea of surrogate gradients \cite{devolder2014first}, in which the gradient used for planning is replaced by a surrogate estimator that does not correspond to the true gradient obtained by differentiating through the physical rollout. This introduces a controlled bias in exchange for improved gradient information in regimes where the true gradient is uninformative. 
The holonomic constraint in \eqref{eq:holonomic-constraint} is active only when $c_i(q)=0$, at which point a nonzero stiffness matrix $D_i$ is used in \eqref{eq:dynamics-first-stage} to enforce the constraint. This switching between inactive and active constraints introduces both discontinuities and vanishing gradients. To smooth this transition, we redefine $D_i$ as $D_i = D_{i,\max} \exp( -\|c_i(q)\|_2 / \kappa )$,
where $\kappa$ is a smoothing parameter. As $\kappa \to 0$, this formulation approaches the original indicator-like activation behavior. 

The contact constraint in \eqref{eq:contact-constraint}, which encodes non-contact, stiction, and sliding regimes, similarly introduces discontinuous and vanishing gradients. We smooth this by relaxing the complementarity condition in \eqref{eq:contact-constraint} to $( J_i(q)\, (q^+ - q) + \begin{bmatrix} \phi_i(q) & 0 & 0 \end{bmatrix})^\top
    \lambda_i
    = \kappa$, which recovers the original complementarity constraint when $\kappa = 0$. 
The smoothed dynamics is used solely for gradient computation and the nonsmooth dynamics are used for forward rollouts, preserving model accuracy.
\vspace{-5pt}
\subsection{Simulator-Informed GPU-Parallel Output Feedback SLS}\label{sec:gpu_sls}
\vspace{-3pt}
To plan robust output feedback policies for the deformable object, we wish to approximately solve the idealized optimization \eqref{eq:output_feedback_ideal} to obtain the output-feedback policy $\boldsymbol{\pi} := (\pi_0 (\cdot), \ldots \pi_{N - 1}(\cdot))$,

\vspace{-15pt}
\begin{equation}\label{eq:output_feedback_ideal}
    \begin{aligned}
        \min_{\boldsymbol{\pi}} \quad & J(\boldsymbol{\pi}) \\
        \text{s.t.} \quad 
        & q_{k+1} = f^*(q_k, u_k) + E(q_k) w_k,  && \forall k \in [N], \qquad q_0 \in \mathcal{Q}_0,\\
        & y_{k+1} = h(q_k) + F(q_k)e_k, && \forall k \in [N],
        \qquad u_k = \pi_k(y_0,\ldots,y_k), && \forall k \in [N],\hspace{-5pt} \\
        & (q_k, u_k) \in \mathcal{S}, && \forall w_k \in \mathcal{B}^{n_q}, \;
        \forall e_k \in \mathcal{B}^{n_y}, \;
        \forall k \in [N].
    \end{aligned}
\end{equation}
where $[N] := \{0,\ldots,N-1\}$, $\mathcal{Q}_0$ denotes the set of possible initial conditions, $
\mathcal{S}
=
\left\{
(q,u) \in \mathbb{R}^{n_q+n_u}
\;\middle|\;
c_i^\top (q,u) + b_i \le 0,\;
i = 1,\ldots,n_c
\right\}$ defines the safe set. While we consider nonlinear constraints in Sec. \ref{sec:results}, to simplify notation, we describe the method for the case of affine constraints. Nonlinear constraints can be robustly enforced following the approach of \cite{zhan2025robustly}. The dynamics and observations are as defined in \eqref{eq:system_real}. 
To efficiently solve \eqref{eq:output_feedback_ideal}, we extend GPU-SLS \cite{fang2026safe}, which is limited to state-feedback (i.e., requires perfect state knowledge), to the output-feedback setting by building on \cite{leeman2026vision}. Following \cite{leeman2026vision}, the output-feedback SLS synthesis problem can be written as,

\vspace{-17pt}
\begin{subequations} \label{eq:of_optimization}\small
    \begin{align}
        \hspace{-10pt}\min_{\mathbf{z}, \mathbf{v}, \mathbf{\Phi}, \boldsymbol{\tau}} \quad & J(\mathbf{z}, \mathbf{v}, \boldsymbol{\tau}, \mathbf{\Phi}) \\[-2pt]
        \hspace{-10pt}\text{s.t.} \quad & z_{k+1} = f(z_k, v_k), \qquad \forall k \in [N] \\
        \hspace{-10pt}& \hspace{-4pt}\begin{bmatrix}\mathbf{I - ZA(z,v)},\ \mathbf{-ZB(z,v)}\end{bmatrix} \mathbf{\Phi} = \begin{bmatrix}\mathbf{I},\ \mathbf{0} \end{bmatrix}, 
        \ \mathbf{\Phi} \begin{bmatrix}
            (\mathbf{I - ZA(z,v)})^\top\hspace{-6pt},\  
            \mathbf{(-ZC(z,v))^\top}
        \end{bmatrix}^\top = \begin{bmatrix}
            \mathbf{I},\
            \mathbf{0}
        \end{bmatrix}^\top\hspace{-8pt}, \label{eq:phi_equation}\\
        \hspace{-10pt}& \textstyle\sum_{j=0}^{k-1} \left\|c_i^\top \mathbf{\Phi}_{k,j}^{\text{w}} E_j \right\|_{2,r} \hspace{-3pt}+ \left\|c_i^\top \mathbf{\Phi}_{k,j}^{\text{e}} F_j \right\|_{2,r} \hspace{-3pt}+ c_i^\top \begin{bmatrix}
            z_k \\ v_k
        \end{bmatrix} \leq -b_i, \qquad \forall i \in [n_c], \; \forall k \in [N]\label{eq:of_optimization_tightening1} \\
        \hspace{-10pt}& \textstyle\sum_{j=0}^{k-1} \left\|\mathbf{\Phi}_{k,j}^{\text{w}} E_j \right\|_{2,r} + \left\|\mathbf{\Phi}_{k,j}^{\text{e}} F_j \right\|_{2,r} \leq \tau_k, \qquad \forall k \in [N]\label{eq:of_optimization_tightening2}
    \end{align}
\end{subequations}
\vspace{-15pt}

\noindent where $\mathbf{A(z, v), B(z,v)}$ and $\mathbf{C(z, v)}$ denote the linearizations of the dynamics and observation models about the nominal trajectory $\mathbf{(z,v)}$, $\|\cdot\|_{2,r}$ denotes the row-wise two norm, and $E_j$ and $F_j$ are the corresponding linearized state and observation disturbance scalings. The matrix \(\mathbf{Z}\) denotes the block downshift operator, with \(I_{n_q}\) on the first subdiagonal and zeros elsewhere. The closed-loop response operators $\mathbf{\Phi}^\text{w}$ and $\mathbf{\Phi}^\text{e}$ characterize the propagation of process and observation disturbances, respectively, via constraint \eqref{eq:phi_equation}, and $\tau_k$ defines the reachable tube width. \eqref{eq:of_optimization_tightening1}--\eqref{eq:of_optimization_tightening2} compute output-feedback reachable tubes under disturbances $w_k$ and $e_k$, and use them to tighten the constraints so the closed-loop reachable set remains inside $\mathcal{S}$. 
To solve \eqref{eq:of_optimization}, \cite{leeman2026vision} decomposes the optimization into three stages: 1) a nominal trajectory generation $\mathbf{(z, v)}$, 2) state feedback controller optimization $\mathbf{\bar\Phi}$, and lastly 3) an output feedback controller optimization $\mathbf{\hat\Phi}$. The first two stages are handled directly by GPU-SLS \cite{fang2026safe}. We therefore focus on the parallel implementation of the third stage by expressing the observer recursions as associative operators, enabling GPU-parallel prefix scans. An optimal solution of the output controller can be found by one backward Kalman recursion,
\begin{equation} \label{eq:kalman_recursions}
    \begin{aligned} 
        \Pi_{k, 0} = \Xi \Xi^\top;\quad
        L_{k, j+1} &= - ( F_j F_j^\top + C_j \Pi_{k,j} C_j^\top)^{-1}C_j \Pi_{k,j}A_j^\top \\
        \Pi_{k, j+1} &= E_j E_j^\top + A_j \Pi_{k,j}A_j^\top + A_j \Pi_{k,j} C^\top_j L_{k, j+1},
    \end{aligned}
\end{equation}
where $\Xi$ denotes the initial state uncertainty, and $N$ independent forward propagations,
\begin{equation} \label{eq:backward_propagations}
    \hspace{-10pt}\begin{aligned}
        &\hat{\mathbf{\Phi}}^x_{k, k} = I; & 
        &\hat{\mathbf{\Phi}}^y_{k, j} = \hat{\mathbf{\Phi}}^x_{k, j} L^\top_{k,j},
        &\hat{\mathbf{\Phi}}^x_{k, j} = \hat{\mathbf{\Phi}}^x_{k, j + 1} ( A_j + L_{k,j}^\top C_j), \quad \forall k \in [N], \ \ \forall j \in [k],
    \end{aligned}
\end{equation}
The matrices $\mathbf{\Phi}^{\text{w}}$ and $\mathbf{\Phi}^{\text{e}}$ can then be assembled by using $\mathbf{M = I - ZA}$ and
\begin{equation}\small
    \hspace{-8pt}\begin{aligned}
        &\mathbf{\Phi}^{\text{xw}} = \bar{\mathbf{\Phi}}^\text{x} + \hat{\mathbf{\Phi}}^\text{x} - \bar{\mathbf{\Phi}}^\text{x} \mathbf{M} \hat{\mathbf{\Phi}}^\text{x}, && \mathbf{\Phi}^{\text{uw}} = \bar{\mathbf{\Phi}}^\text{u} \mathbf{(I - M \hat{\Phi}^x)}, 
        & \mathbf{\Phi}^{\text{xe}} = \mathbf{\hat{\Phi}}^\text{y} - \bar{\mathbf{\Phi}}^\text{x} \mathbf{M} \hat{\mathbf{\Phi}}^\text{y}, && \mathbf{\Phi}^\text{ue} = - \bar{\mathbf{\Phi}}^\text{u} \mathbf{M} \hat{\mathbf{\Phi}}^\text{y}.
    \end{aligned}\hspace{-2pt}
\end{equation}
Both \eqref{eq:kalman_recursions} and \eqref{eq:backward_propagations} require dense matrix multiplications serially through the planning horizon. As such, we use parallel associative scans on the GPU. We define the conditional value function, $V_{i \rightarrow l} (\mathbf{\Phi}_{k, i}, \mathbf{\Phi}_{k,l})\!=\!\max_\gamma \tfrac{1}{2} \mathbf{\Phi}_{k,i}^\top \tilde P_{i, l} \mathbf{\Phi}_{{k, i}} - \tfrac{1}{2} \gamma^\top \tilde{D}_{i,l} \gamma - \gamma ^\top (\mathbf{\Phi}_{k,l} - \tilde{A}_{i,l} \mathbf{\Phi}_{k, i})$, and combination rules,
\begin{equation} \label{eq:combination_rules}
    \begin{aligned}
        \tilde P_{i, l}
        &=
        \tilde P_{i,m} \otimes \tilde P_{m, l}
        := 
        ( \tilde A_{i, m}^\top
        ( I + \tilde{P}_{m,l} \tilde D_{i,m})^{-1} \tilde{P}_{m,l} \tilde{A}_{i,m}
        +
        \tilde P_{i,m}, \\[-2pt]
        \tilde{A}_{i,l}
        &=
        \tilde{A}_{i,m} \otimes \tilde{A}_{m,l}
        :=
        \tilde A_{m, l} ( I + \tilde{D}_{i,m} \tilde P_{m,l})^{-1} \tilde A_{i,m}, \\[-2pt]
        \tilde{D}_{i,l}
        &=
        \tilde{D}_{i,m} \otimes \tilde{D}_{m, l}
        := 
        \tilde{A}_{m,l}( I + \tilde{D}_{i,m} \tilde{P}_{m,l})^{-1} \tilde{D}_{i,m} ( \tilde{A}_{m,l})^\top
        +
        \tilde{D}_{m,l},
    \end{aligned}
\end{equation}
with the initial values as
\begin{equation} \label{eq:initialization}\small
    \begin{gathered}
        \tilde{A}_{i + 1, i} = A_i^\top, \ \ \
        \tilde{D}_{i + 1, i} = C_i^\top(F_iF_i^\top )^{-1} C_i, \ \ \ 
        \tilde{P}_{i + 1, i} = E_iE_i^\top, \ \ \
        \tilde{P}_{1,0} = \Xi \Xi^\top, \ \ \
        \tilde{A}_{1,0} = 
        \tilde{D}_{1,0} = 0.
    \end{gathered}
\end{equation}
We can then recover Kalman gains $L_{k,j}$ through a forward parallel associative scan using \eqref{eq:combination_rules}--\eqref{eq:initialization} to recover $\Pi_{k,j} = \tilde P_{0, j} := \tilde P_{0, 1} \otimes \tilde P_{1, 2} \otimes \cdots \otimes \tilde P_{j - 1, j}$, and then $L_{k,j}$ from \eqref{eq:kalman_recursions}. To retrieve $\mathbf{\hat \Phi}$, we use a similar procedure and define the conditional observer propagation $\hat{N}^{(k)}_{i \rightarrow l} ( \hat{\mathbf{\Phi}}^x_{k,i}, \hat{\mathbf{\Phi}}^x_{k,l}) =  \hat{\mathbf{\Phi}}^x_{k,i} \hat{A}_{i,l}^{(k)}$ for each disturbance injection $k$
with combination rule $\hat{A}^{(k)}_{m,l}\hat{A}^{(k)}_{i,m}$ and initialization $\hat{A}^{(k)}_{i+1,i} = A_i + L_{k,i}^\top C_i$.  We obtain the solution to $\eqref{eq:backward_propagations}$ by a forward parallel associative scan, $\hat{\mathbf{\Phi}}^x_{k,i} = \hat{N}_{j + 1 \rightarrow j}^{(k)} \otimes \hat{N}_{j \rightarrow j - 1}^{(k)} \otimes \cdots \otimes \hat{N}_{i + 1 \rightarrow i}^{(k)}$, and $\hat{\mathbf{\Phi}}^y_{k, j} = \hat{\mathbf{\Phi}}^x_{k, j} L^\top_{k,j}$, solving the observer-dependent closed-loop response in logarithmic time, analogous to the complexity of GPU-SLS.
\vspace{-5pt}

\subsection{Accelerating Model-Based RL via Analytical Policy Gradients (APG)}\label{sec:mbrl}
\vspace{-4pt}

The differentiable simulator proposed in Sec. \ref{sec:simulator} can be used to support model-based reinforcement learning (RL) in multiple ways; we point out two ways here. First, it can generate expert trajectories to guide policy learning. This is especially beneficial in contact-rich deformable manipulation, where sparse rewards and intermittent contact can make exploration from scratch prohibitively sample-inefficient. Second, by incorporating the contact-smoothed simulator as a differentiable layer \eqref{eq:diff-sim}, direct policy gradients can guide policies toward discovering and maintaining contact with reduced reliance on stochastic exploration, improving the stability and sample efficiency of policy learning. We denote this model-based RL (MBRL) variant of our method as APG, as it leverages \underline{a}nalytical \underline{p}olicy \underline{g}radients to accelerate policy learning, and is evaluated in Sec. \ref{sec:results}.
\vspace{-5pt}

\subsection{Calibrated Reachability Guarantees via Conformal Prediction (CP)}\label{sec:conformal}
\vspace{-3pt}
The SLS formulation in Sec.~\ref{sec:gpu_sls} requires uncertainty sets for the closed-loop dynamics and observation model to certify safety for the real system~\eqref{eq:system_real}. On hardware, however, independent labels for the true process disturbance or state-estimation error are typically unavailable. Instead, we measure the discrepancy between the model-predicted next observation and the next observation returned by perception, and use CP to calibrate an aggregate one-step prediction error in the controller coordinates. 
In our implementation (Sec.~\ref{sec:results}), the perception estimates the full rope state from RGB-D images. Thus, in~\eqref{eq:measurements_real}, we take $h(q)=q$, so $y_{k+1}=q_k+e_k$ with $\|e_k\|_2\le\rho$, where $\rho$ is a worst-case measurement-error bound from camera accuracy specifications. We define the perceived state $\bar q_k:=y_{k+1}$ as the noisy full-state estimate used by the controller. Although the physical system evolves on latent $q_k$, our SLS implementation plans over $\bar q_k$; the CP-calibrated residual certifies tubes for $\bar q_k$, which are conservatively inflated by $\rho$ to obtain true-state tubes.

Let $\mathcal{D}_{\mathrm{cal}}:=\{(\bar q_i,u_i,\bar q_i^+)\}_{i=1}^n$ be a held-out calibration set of perceived-state transitions, where $\bar q_i^+$ is the next perceived state; equivalently, a sample is $(y_{k+1},u_k,y_{k+2})$. We define nonconformity scores
$s_i:=\|\bar q_i^+-f(\bar q_i,u_i)\|_2$,
and let $s_{(1)}\le\cdots\le s_{(n)}$ be the sorted scores. For trajectory-level miscoverage $\delta$ over horizon $T$, set $\bar\delta:=\delta/T$ and
$q:=s_{(\lceil (n+1)(1-\bar\delta)\rceil)}$,
with $q=+\infty$ if the index exceeds $n$. The CP set for the next perceived state is
$\mathcal{C}(\bar q,u)
:=
\left\{
\bar q^+\in\mathcal{Q}
\;\middle|\;
\|\bar q^+-f(\bar q,u)\|_2\le r_\textrm{CP}
\right\}$. 
We instantiate SLS on the perceived-state surrogate
$\bar q_{k+1}=f(\bar q_k,u_k)+r_\textrm{CP}I w_k$, with $w_k\in\mathcal{B}^{n_q}$. This is a calibrated model of the full-state estimate used by the controller, rather than an independently identified process-noise model for the latent state $q_k$. Let $\bar{\mathcal{T}}_k^q$ and $\mathcal{T}_k^u$ denote the resulting perceived-state and control reachable tubes. Assuming exchangeability between closed-loop residual scores $\|\bar q_{k+1}-f(\bar q_k,u_k)\|_2$ and the calibration scores, CP and robust SLS ensure (proof in App. \ref{app:proofs}):
\vspace{-1pt}
\begin{theorem}[True-state containment from CP tubes]
\label{thm:perceived_to_true_tubes}
Let $y_{k+1}=q_k+e_k$ with $\|e_k\|_2\le\rho$, and define
$\bar q_k:=y_{k+1}$. Suppose conformal calibration and robust SLS synthesis yield
perceived-state and control tubes $\bar{\mathcal{T}}_k^q$ and $\mathcal{T}_k^u$
satisfying
$\mathbb{P}
    \left[
        \bar q_k\in\bar{\mathcal{T}}_k^q,\;
        u_k\in\mathcal{T}_k^u,\;
        k=0,\ldots,N
    \right]
    \ge 1-\delta$. 
Then the true state and control trajectories satisfy
$\mathbb{P}
    \left[
        q_k\in
        \bar{\mathcal{T}}_k^q\oplus \rho\mathbb{B}^{n_q},\;
        u_k\in\mathcal{T}_k^u,\;
        k=0,\ldots,N
    \right]
    \ge 1-\delta$. 
\end{theorem}
\vspace{-7.5pt}
Accordingly, we define the true-state tube as
$\mathcal{T}_k^q
    :=
    \bar{\mathcal{T}}_k^q\oplus \rho\mathcal{B}^{n_q}$. 
Thus, to enforce true-state safety constraints
$q_k\in\mathcal{Q}_{\mathrm{safe}}$, it suffices to impose the tightened
perceived-state condition $\bar{\mathcal{T}}_k^q
    \subseteq
    \mathcal{Q}_{\mathrm{safe}}\ominus \rho\mathcal{B}^{n_q}$. 
State-dependent uncertainty bounds can be incorporated by replacing the scalar radius
$r_\textrm{CP}$ with a calibrated state-dependent bound, e.g., by adapting
\cite{srinivasan2026safety} to the perceived-state setting.

\vspace{-6pt}
\section{Results}\label{sec:results}
\vspace{-4pt}
We compare CORD-SLS and APG against MPPI \cite{williams2017model} and model-free RL methods (namely, PPO \cite{schulman2017proximal}) on deformable manipulation, empirically validate the robustness of our output-feedback solver, and evaluate the efficacy of our method on hardware.

\begin{table}[t]
\centering
\small
\begin{tabular}{llrccc}
\toprule
Task & Method & \makecell{Dynamics\\Evaluations Count} & \makecell{Execution Time\\Per-Step [s]} & Safety Rate [\%] & Goal Error [m] \\
\midrule
\multirow{4}{*}{Lift Rope}
  & Ours  &     5400 & 0.050  $\pm$ 0.001  & 100.0 $\pm$ 0.0 & 0.0548 $\pm$ 0.0003 \\
  & MPPI  &  1024000 & 0.177  $\pm$ 0.043  & 100.0 $\pm$ 0.0 & 0.0681 $\pm$ 0.0038 \\
  & PPO   & 16000000 & 0.0004 $\pm$ 0.0001 &  96.5 $\pm$ 0.2 & 0.0587 $\pm$ 0.0016 \\
  & APG   & 16000000 & 0.0004 $\pm$ 0.0001 & 100.0 $\pm$ 0.0 & 0.0573 $\pm$ 0.0017 \\
\midrule
\multirow{4}{*}{Drag Rope}
  & Ours  &    13500 & 0.070  $\pm$ 0.012  & 99.9 $\pm$ 0.1 & 0.0360 $\pm$ 0.0077 \\
  & MPPI  &  2560000 & 0.234  $\pm$ 0.018  &  1.0 $\pm$ 0.1 & 0.0814 $\pm$ 0.0034 \\
  & PPO   & 40000000 & 0.0004 $\pm$ 0.0001 & 97.0 $\pm$ 0.4 & 0.0320 $\pm$ 0.0092 \\
  & APG   & 40000000 & 0.0004 $\pm$ 0.0001 & 93.0 $\pm$ 0.4 & 0.0302 $\pm$ 0.0061 \\
\midrule
\multirow{2}{*}{Fold cloth}
  & Ours  &  3400 & 0.758 $\pm$ 0.107 & 100.0 $\pm$ 0.0 & 0.0611 $\pm$ 0.0072 \\
  & MPPI  & 64000 & 18.81 $\pm$ 6.11  & 64.0 $\pm$  6.1 & 0.0738 $\pm$ 0.0079 \\
\midrule
\multirow{2}{*}{Flatten cloth}
  & Ours  &  3400 & 0.711 $\pm$ 0.077 & 100.0 $\pm$ 0.0 & 0.0234 $\pm$ 0.0089 \\
  & MPPI  & 64000 & 11.10 $\pm$ 4.66  &  80.0 $\pm$ 0.2 & 0.0526 $\pm$ 0.0102 \\
\bottomrule
\end{tabular}
\caption{Runtime and performance metrics across tasks.}
\vspace{-20pt}
\label{tab:task_results}
\end{table}
\vspace{-10pt}

\paragraph{Simulation Benchmark Comparison}
 See App. \ref{app:simulator} for baseline comparisons of our differentiable simulator against prior simulators \cite{chen2024differentiable, chen2022daxbench} in terms of accuracy and runtime. We show that our simulator achieves comparable accuracy while accelerating forward simulation and Jacobian computation by up to 700$\times$ and 2000$\times$, respectively. This efficiency, along with the integration of contact smoothing, enables real-time control rates and low goal error in our deformable object planning experiments, which we discuss next.

\paragraph{Planning Benchmark Comparison on Deformable Manipulation Tasks}
We benchmark our approach on four deformable manipulation tasks shown in Fig.~\ref{fig:timelapse}. In the lift rope task, a pair of grippers must establish contact with the rope to activate the holonomic constraint required to lift it. Our method leverages gradients from the dynamics to guide actions toward making contact. In contrast, MPPI requires substantially more dynamics evaluations to identify contact-making actions, resulting in higher execution times. Among the RL methods, APG leverages differentiable dynamics to compute policy gradients, whereas PPO is warm-started with 10 expert demonstrations generated by CORD-SLS to enable it to solve the task within the same interaction budget. Without this warm start, the PPO policy fails to make contact with the rope after the same number of environment interactions, and thus fails to complete the task.

In the drag rope task, the gripper must establish contact with the rope and transport it to a goal location while avoiding an obstacle. CORD-SLS explicitly accounts for disturbances and computes trajectory tubes (Fig.~\ref{fig:drag-rope}, see App. \ref{app:mpc_tubes} for full tube snapshots), enabling robust obstacle avoidance. Although obstacle avoidance is encoded as a penalty in the MPPI cost function, the best sampled trajectories often pass through the obstacle, leading to a very low safety rate. Increasing the obstacle penalty further causes the rope to remain near its initial position. APG finds a lower-cost policy than PPO with fewer simulator samples (Fig. \ref{fig:rl_convergence}b, App. \ref{app:rl_convergence}), but both RL methods exhibit constraint violations because they do not explicitly account for disturbances during execution.

For the cloth folding and flattening tasks, despite the high-dimensional state space (300 DoF), the GPU-parallelized simulator and SLS solver allow each MPC update to be computed in under one second (see App. \ref{app:mpc_tubes} for reachable tubes). MPPI, by comparison, requires significantly more dynamics evaluations to achieve comparable performance, resulting in longer execution times as well as higher goal error and constraint violation rates. We do not include RL baselines for these tasks because the 300-DoF cloth system lies in a regime where online trajectory optimization remains tractable, while learning a neural policy becomes prohibitively sample-intensive.

\begin{figure}
    \centering
    \begin{subfigure}[b]{3.4cm}
        \centering
        \adjincludegraphics[width=1.65cm, trim={{0.17\width} {0.3\height} {0.12\width} {0.35\height}}, clip]{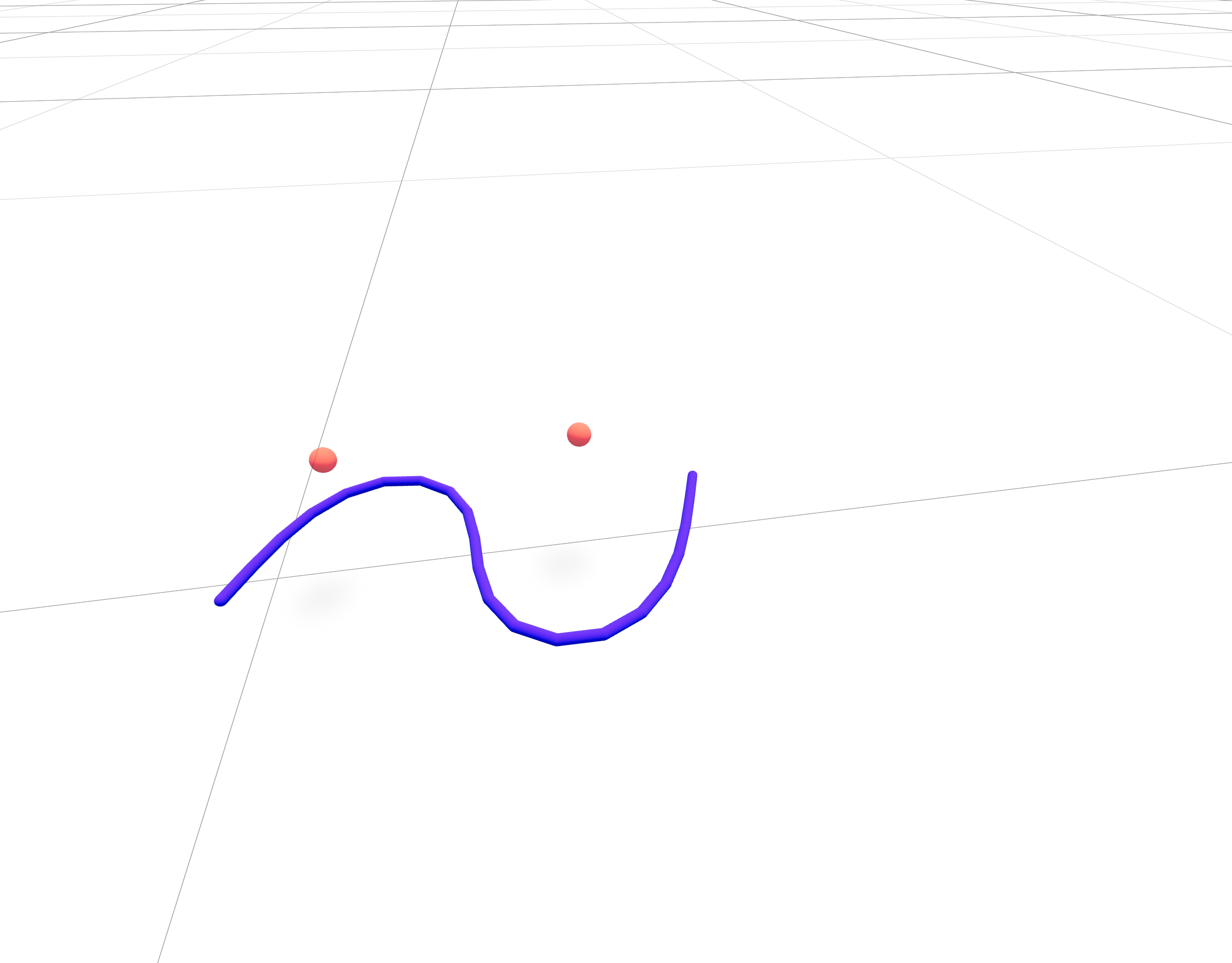}
        \adjincludegraphics[width=1.65cm, trim={{0.17\width} {0.3\height} {0.12\width} {0.35\height}}, clip]{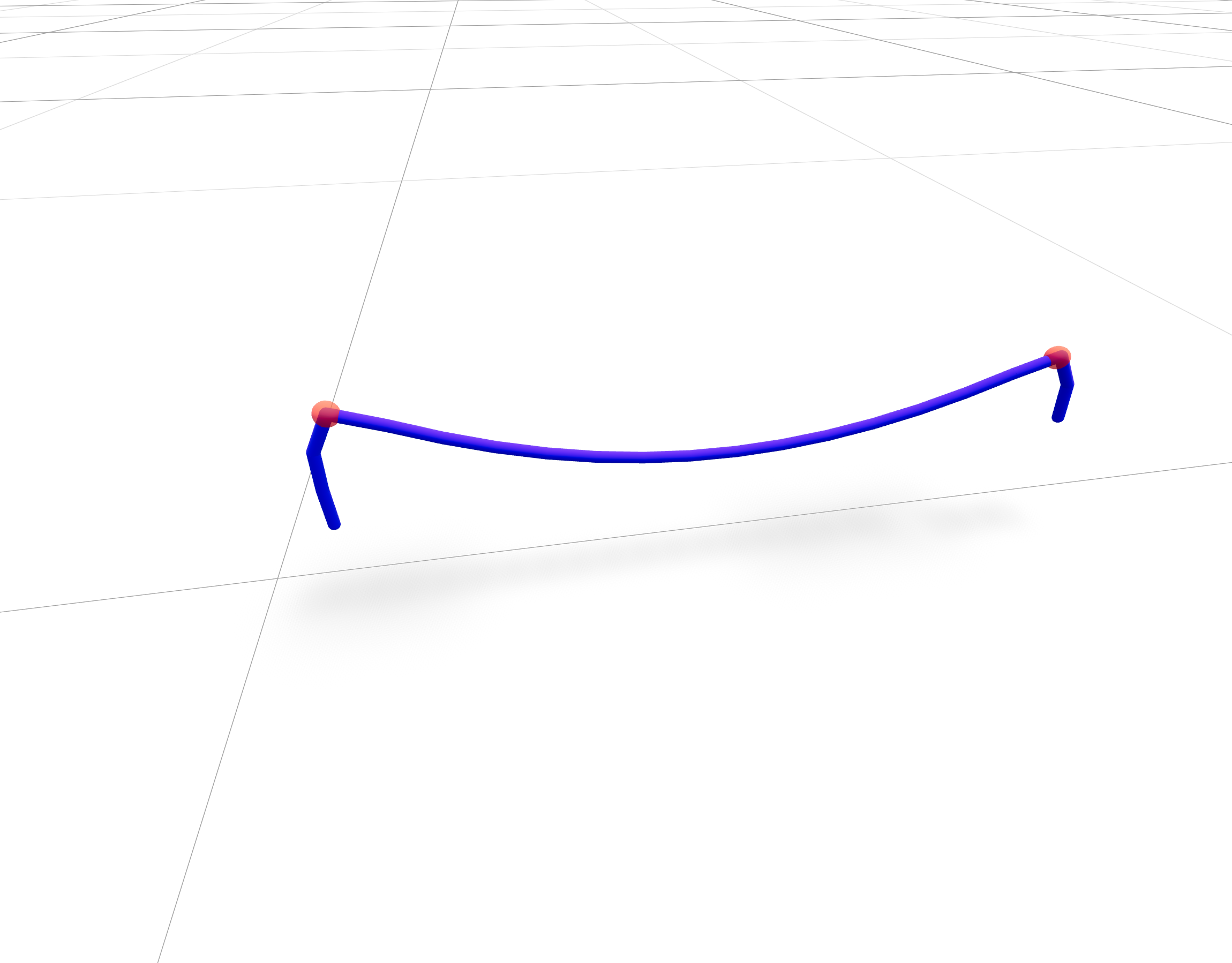}
        \subcaption{}
    \end{subfigure}
    \hfill
    \vsubdividershort
    \hfill
    \begin{subfigure}[b]{10.3cm}
        \centering
        \adjincludegraphics[width=1.65cm, trim={{0.07\width} {0.2\height} {0.07\width} {0.35\height}}, clip]{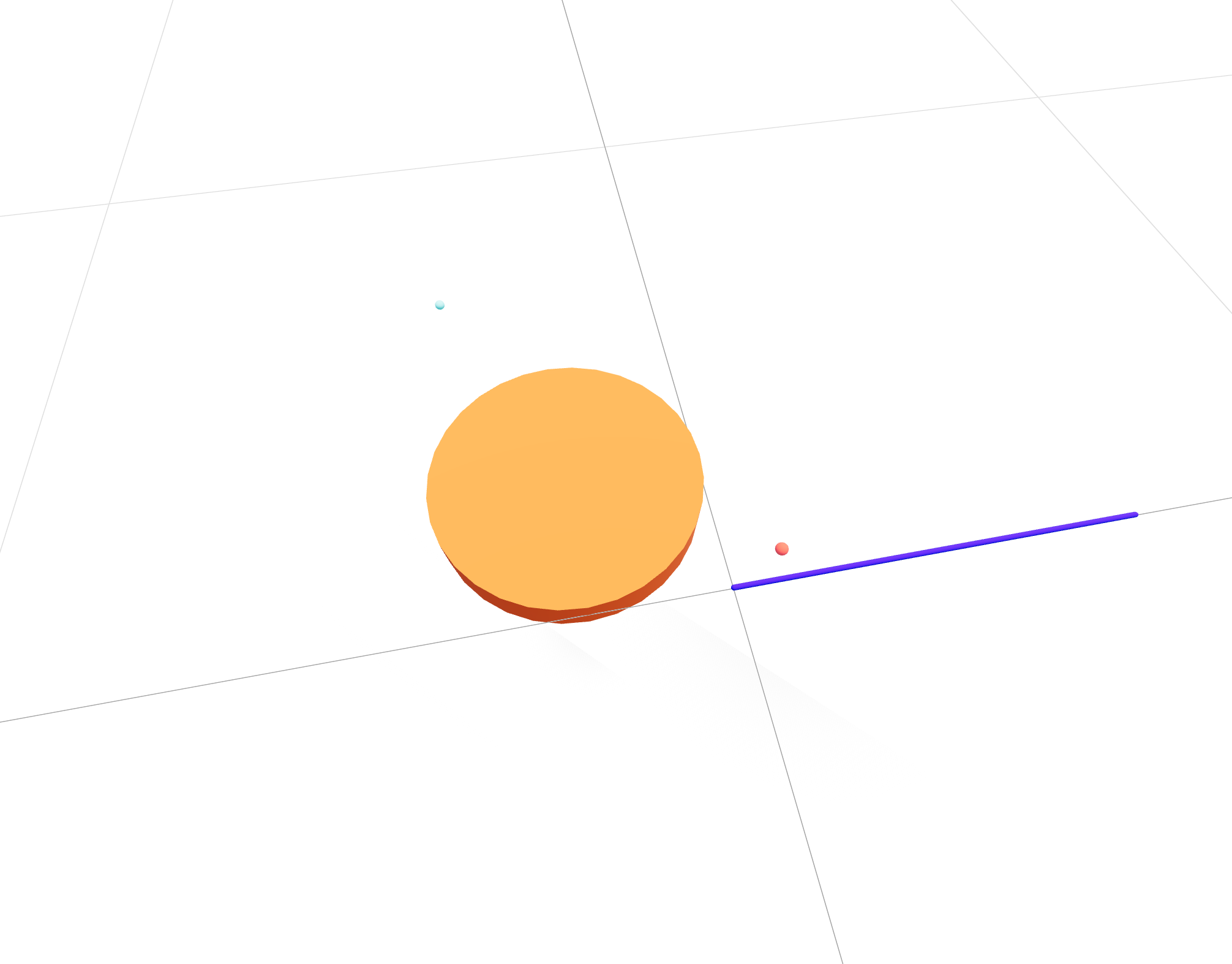}
        \adjincludegraphics[width=1.65cm, trim={{0.07\width} {0.2\height} {0.07\width} {0.35\height}}, clip]{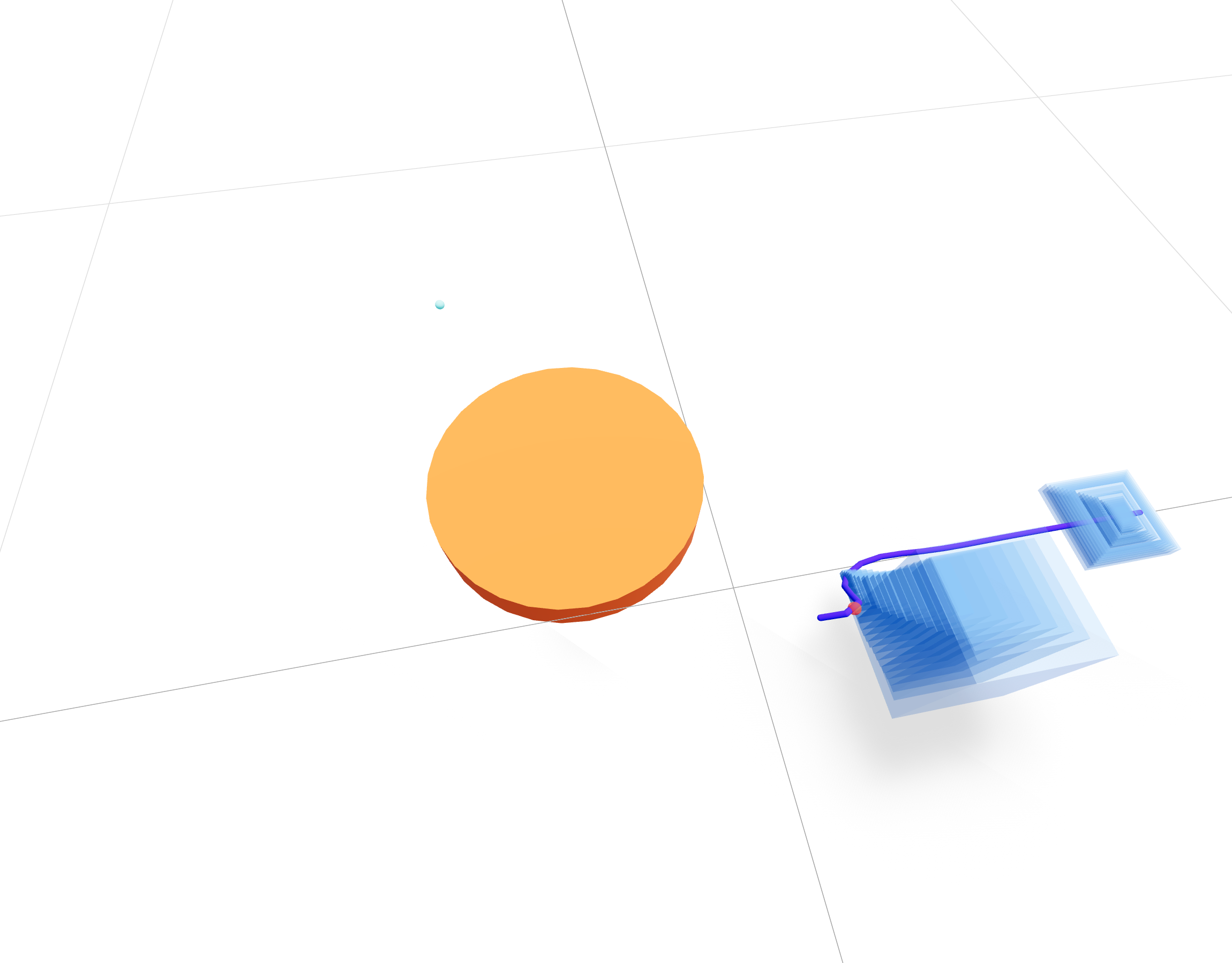}
        \adjincludegraphics[width=1.65cm, trim={{0.07\width} {0.2\height} {0.07\width} {0.35\height}}, clip]{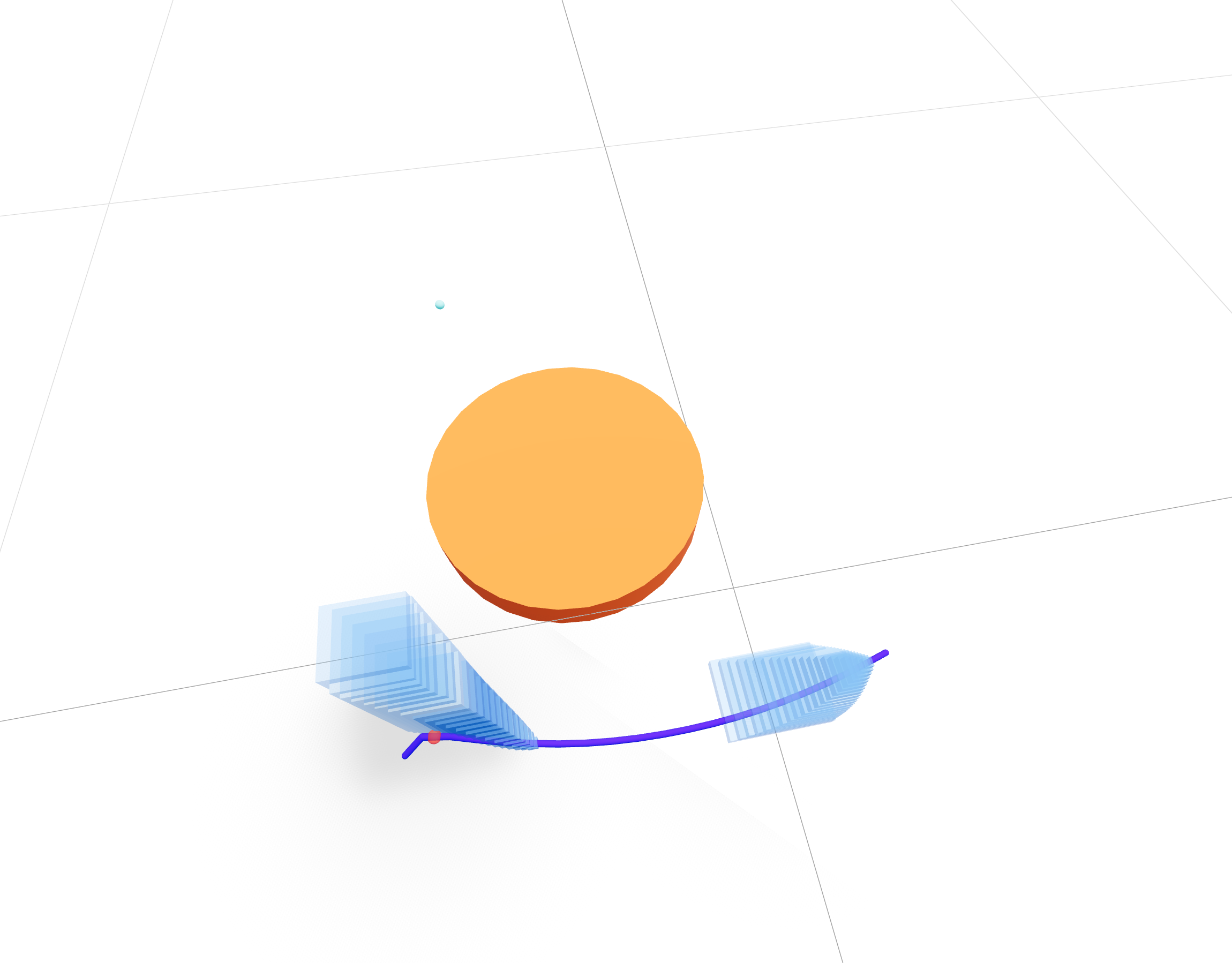}
        \adjincludegraphics[width=1.65cm, trim={{0.07\width} {0.2\height} {0.07\width} {0.35\height}}, clip]{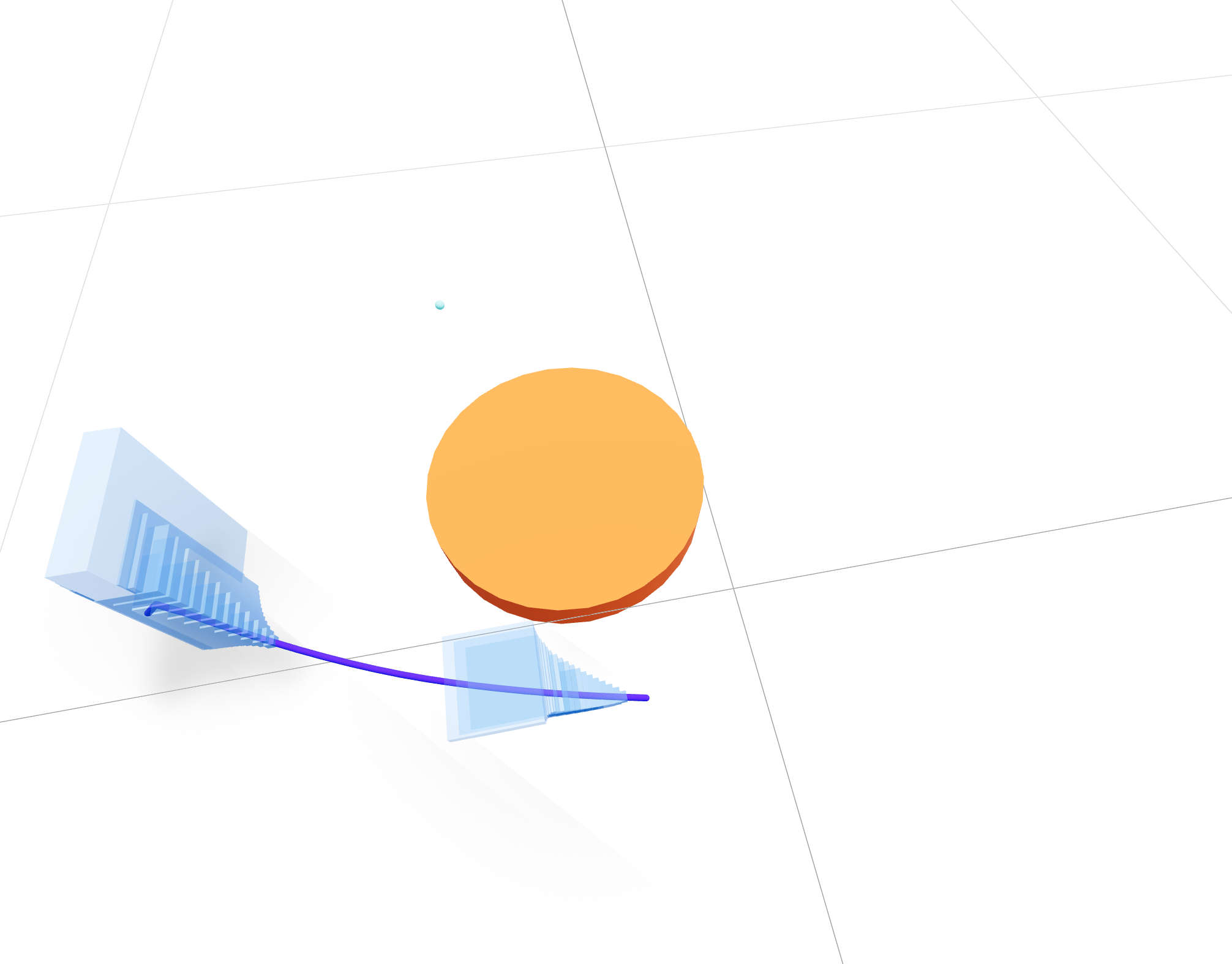}
        \adjincludegraphics[width=1.65cm, trim={{0.07\width} {0.2\height} {0.07\width} {0.35\height}}, clip]{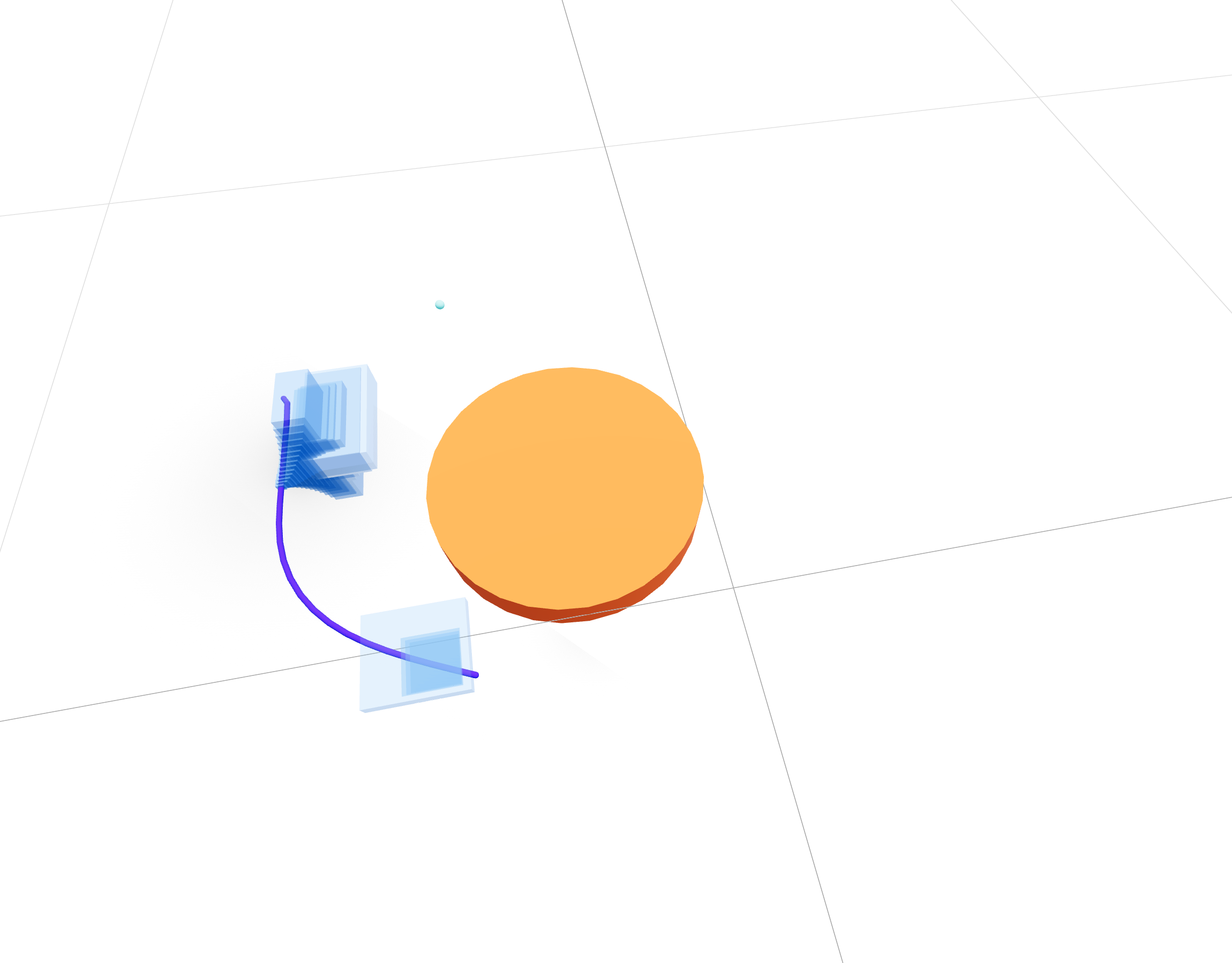}
        \adjincludegraphics[width=1.65cm, trim={{0.07\width} {0.2\height} {0.07\width} {0.35\height}}, clip]{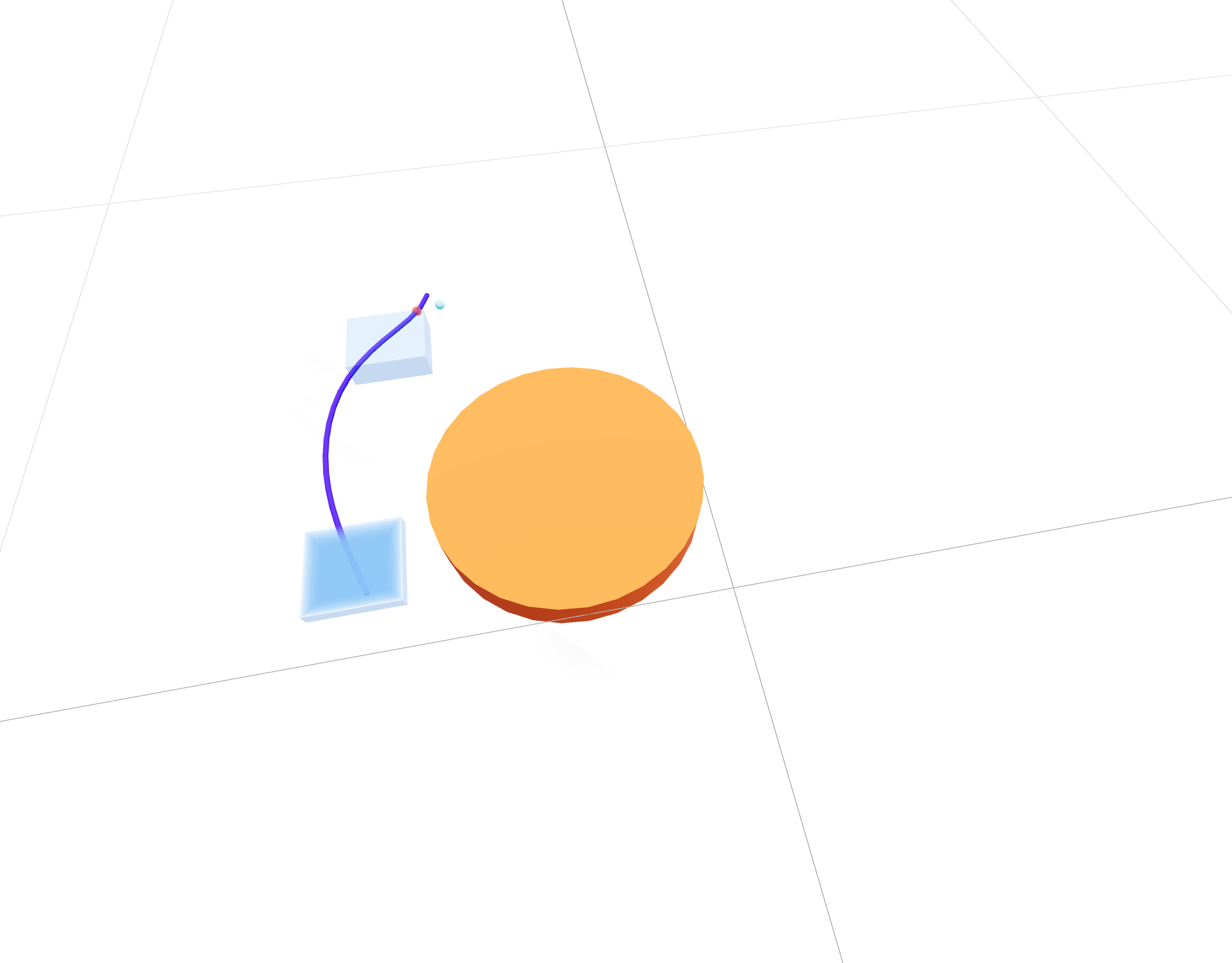}
        \subcaption{}
        \label{fig:drag-rope}
    \end{subfigure}
    \\
    \begin{subfigure}[b]{6.9cm}
        \centering
        \adjincludegraphics[width=1.65cm, trim={{0.23\width} {0.15\height} {0.15\width} {0.1\height}}, clip]{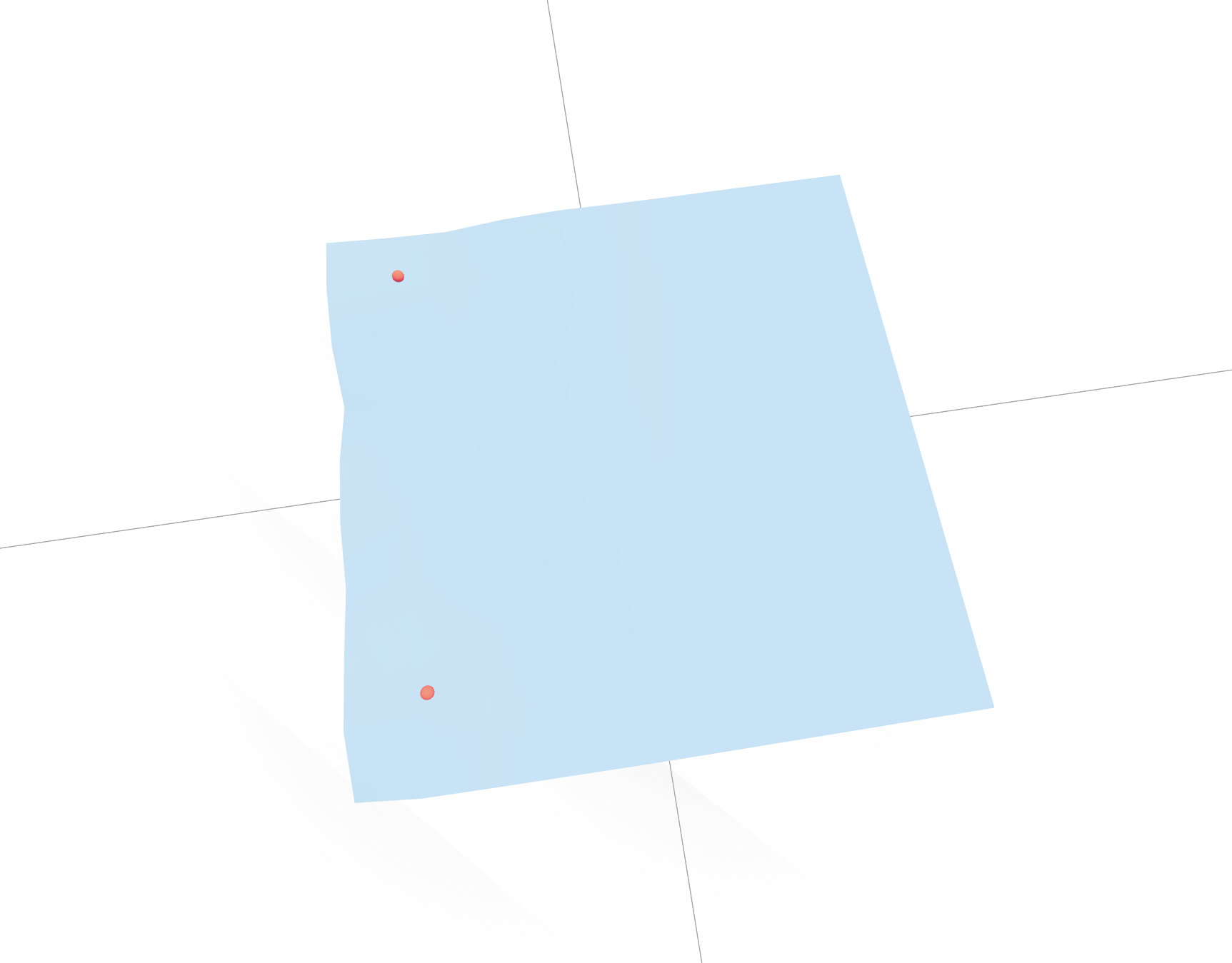}
        \adjincludegraphics[width=1.65cm, trim={{0.23\width} {0.15\height} {0.15\width} {0.1\height}}, clip]{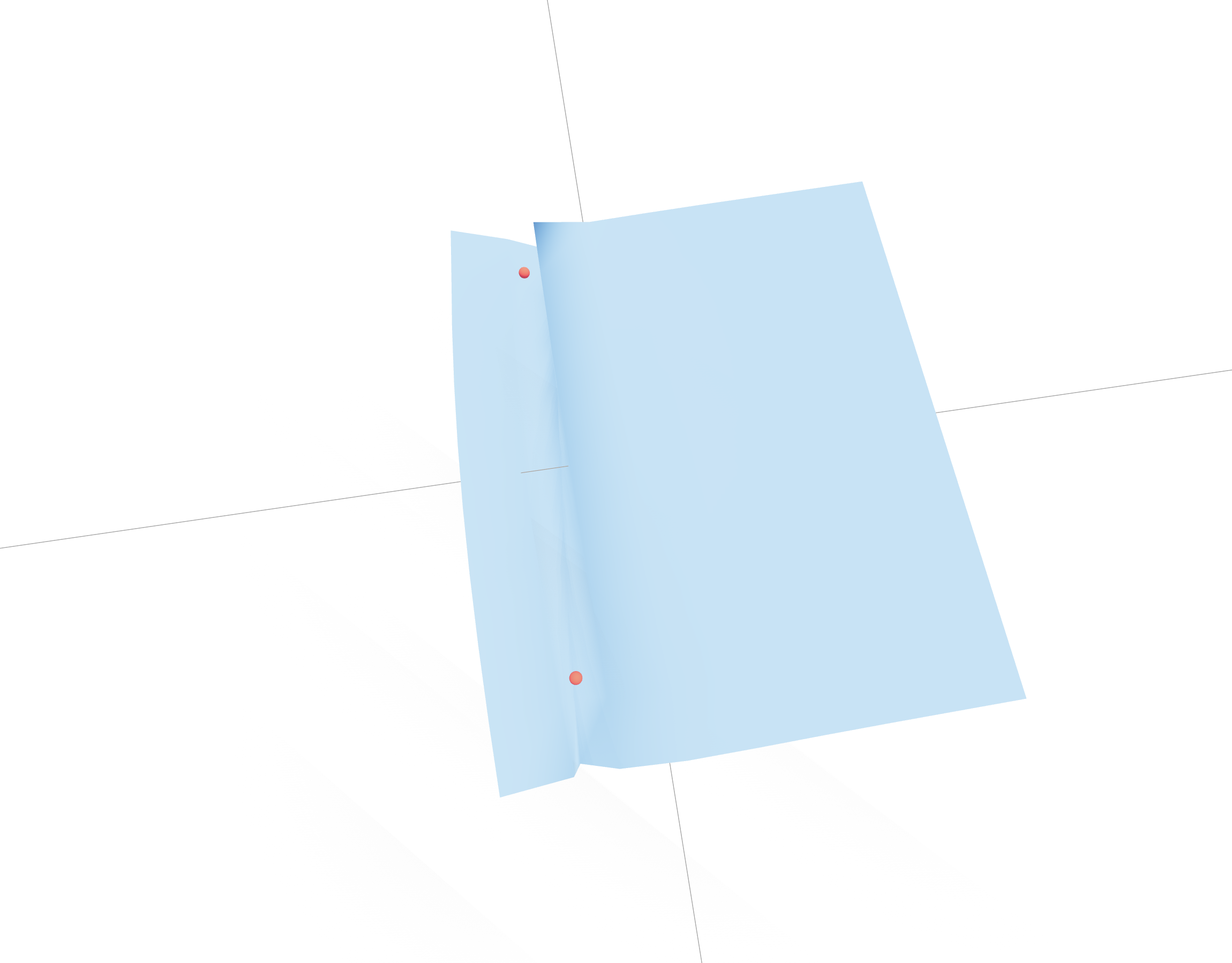}
        \adjincludegraphics[width=1.65cm, trim={{0.23\width} {0.15\height} {0.15\width} {0.1\height}}, clip]{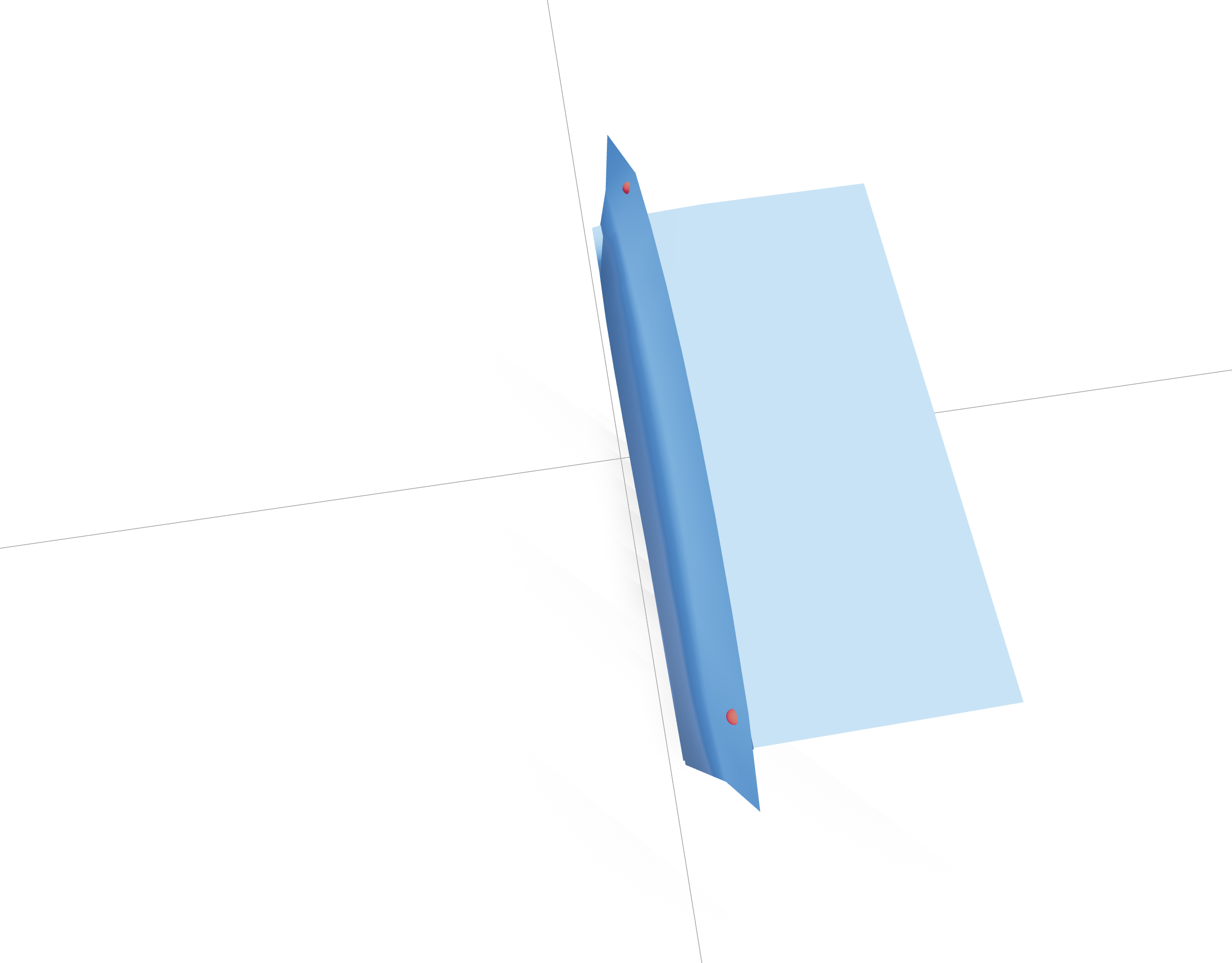}
        \adjincludegraphics[width=1.65cm, trim={{0.23\width} {0.15\height} {0.15\width} {0.1\height}}, clip]{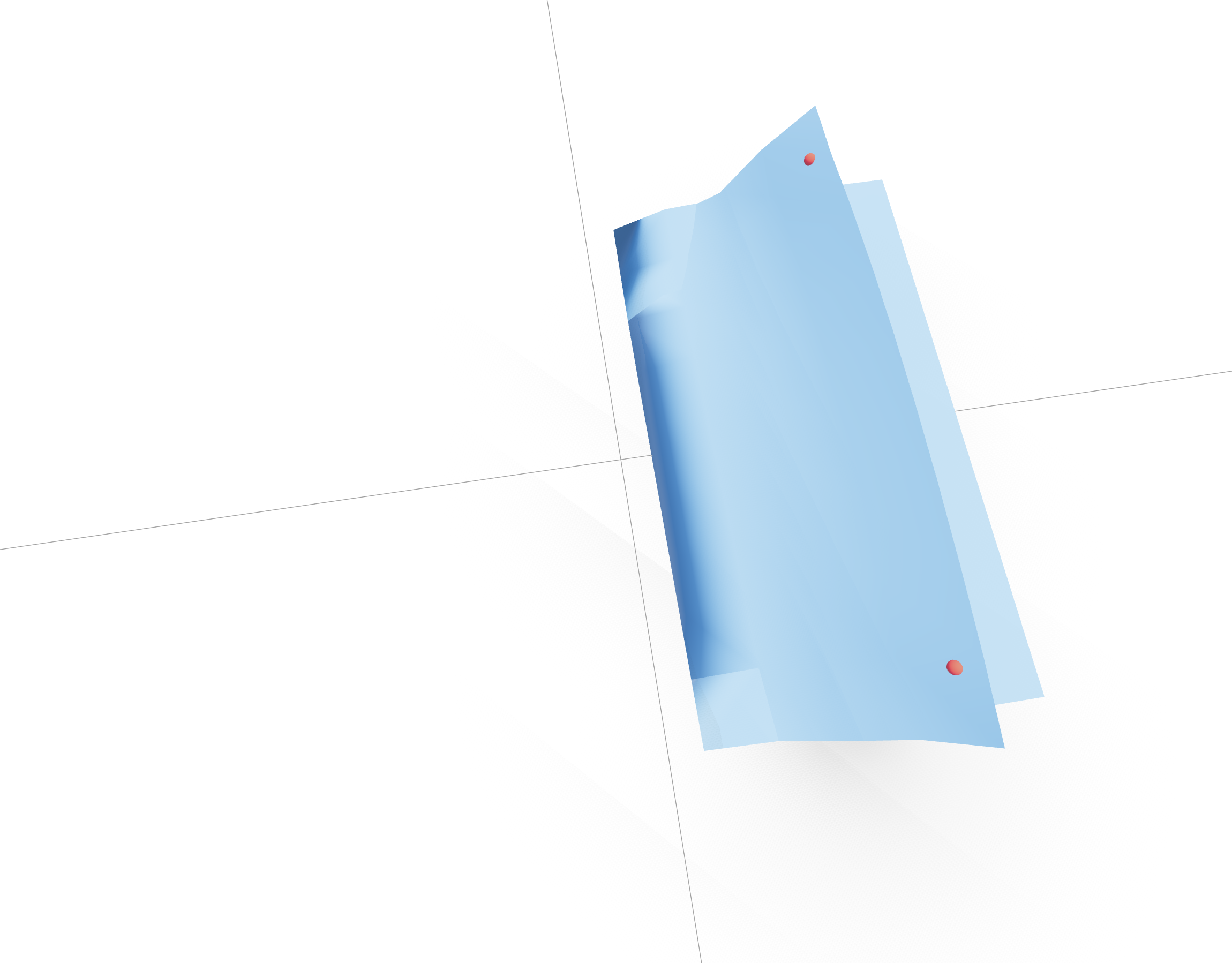}
        \subcaption{}
        \label{fig:fold-cloth}
    \end{subfigure}
    \hfill
    \vsubdivider
    \hfill
    \begin{subfigure}[b]{6.8cm}
        \centering
        \adjincludegraphics[width=1.64cm, trim={{0.23\width} {0.15\height} {0.15\width} {0.1\height}}, clip]{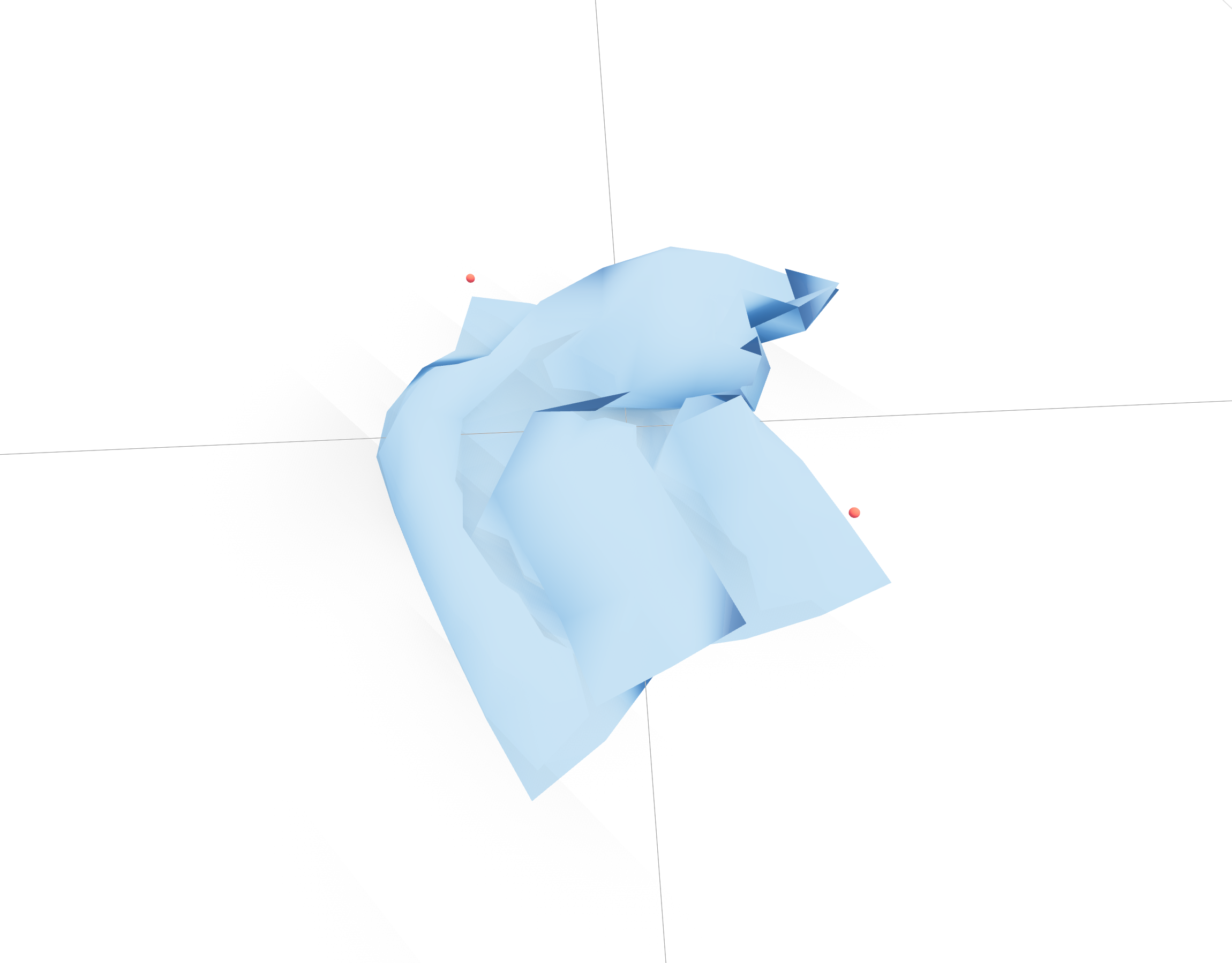}
        \adjincludegraphics[width=1.64cm, trim={{0.23\width} {0.15\height} {0.15\width} {0.1\height}}, clip]{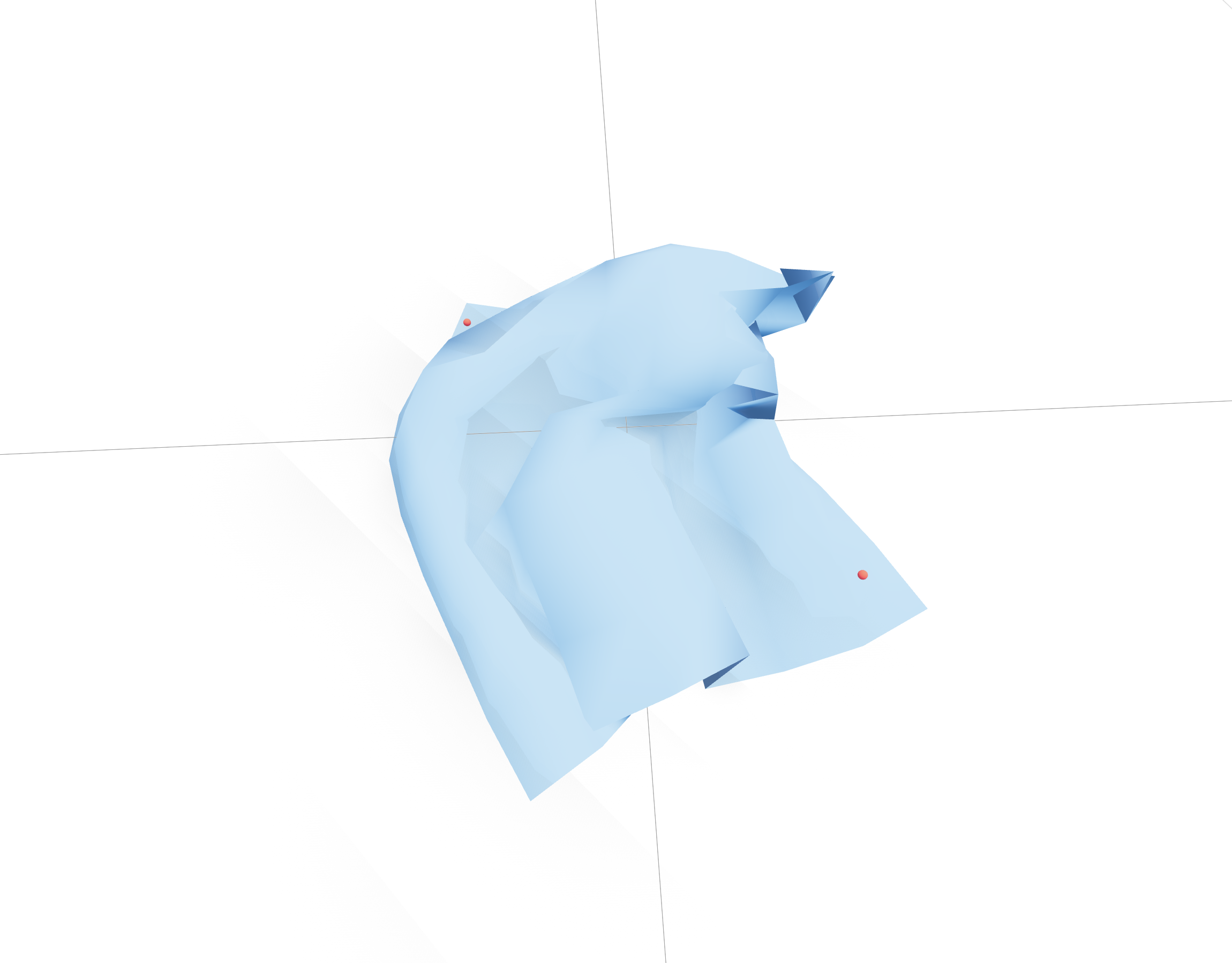}
        \adjincludegraphics[width=1.64cm, trim={{0.23\width} {0.15\height} {0.15\width} {0.1\height}}, clip]{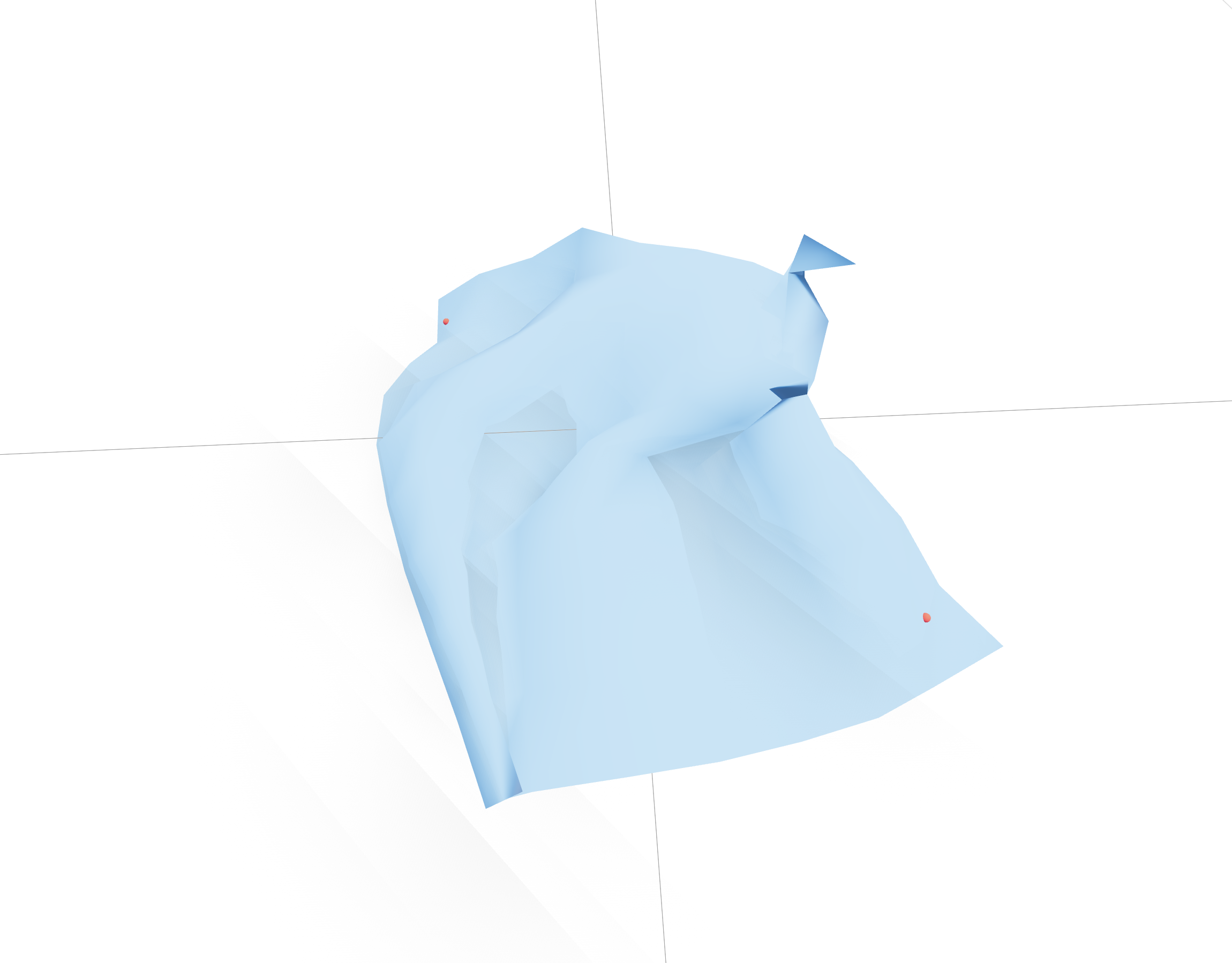}
        \adjincludegraphics[width=1.64cm, trim={{0.23\width} {0.15\height} {0.15\width} {0.1\height}}, clip]{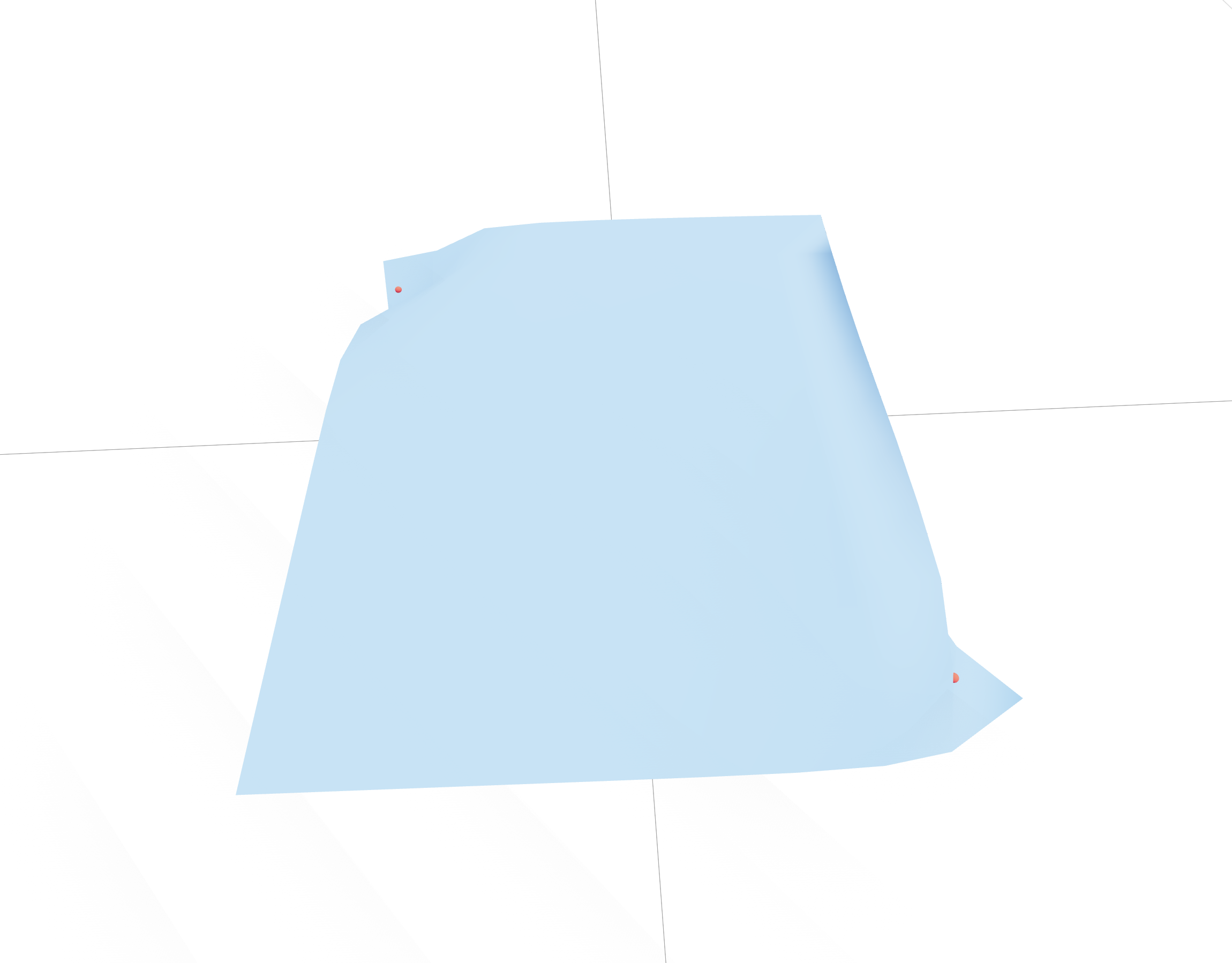}
        \subcaption{}
    \end{subfigure}
    \vspace{-6pt}
    \caption{Timelapse of our method on each task. (a) Rope lifting, (b) Rope dragging around an obstacle with uncertainty tubes, (c) Cloth folding, (d) Cloth flattening. By formulating control as trajectory optimization, our method generalizes across diverse tasks without task-specific retraining.}
    \vspace{-3pt}
    \label{fig:timelapse}
\end{figure}

\begin{figure}
    \centering
    \includegraphics[width=\linewidth]{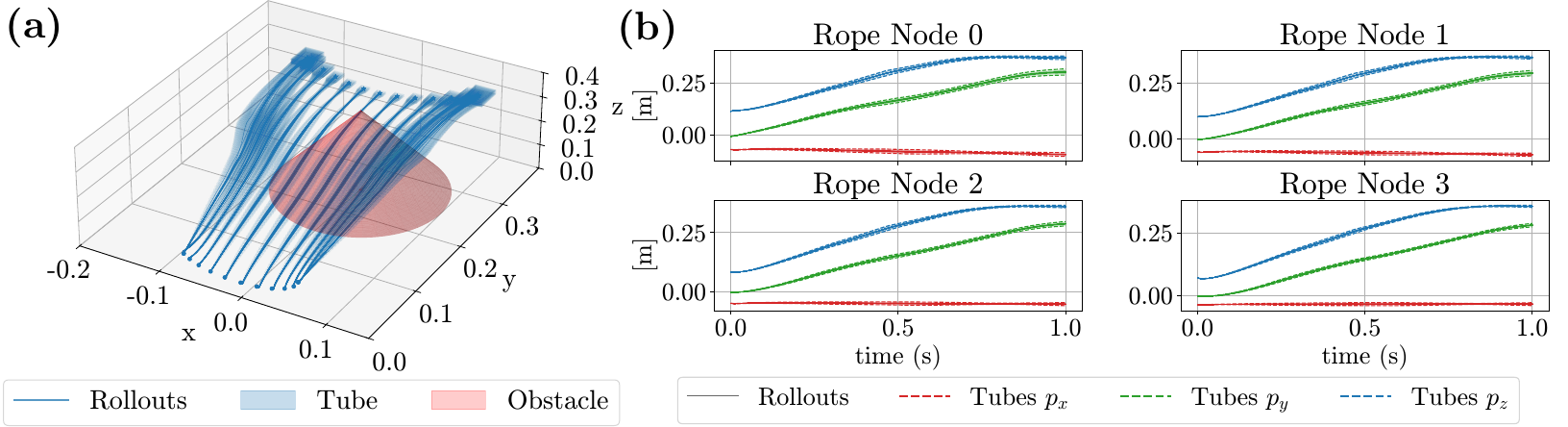}
    \vspace{-12pt}
    \caption{\textbf{(a)}: Solved trajectory for a rope manipulation task using CORD-SLS. \textbf{(b)}: Robust tubes for four rope nodes and their corresponding rollouts. All rollouts remain within their tubes, demonstrating robustness to both dynamical error and measurement uncertainty. See App. \ref{app:output_feedback_tubes} for reachable tubes for all state dimensions.}
    \label{fig:output_feedback}
    \vspace{-10pt}
\end{figure}

\vspace{-10pt}
\paragraph{Output Feedback}
In Fig.~\ref{fig:output_feedback}, we task the controller with robustly navigating a rope over an obstacle with simulated measurement and dynamical noise. In Fig.~\ref{fig:output_feedback}\textbf{(a)} we show the robust tubes and trajectory computed by our solver, which account for dynamical and measurement uncertainty. In Fig.~\ref{fig:output_feedback}\textbf{(b)} we show a subset of the robust tubes for 4 rope nodes and their rollouts, further illustrating the controller remains robust to both measurement and dynamical noise throughout execution. 
\vspace{-12pt}
\paragraph{Hardware} \label{sec:hardware_results}
Lastly, we validate our approach on hardware. As shown in Fig.~\ref{fig:hardware_experiments}, we consider two deformable manipulation tasks: obstacle rope manipulation and cloth folding. For both tasks, we estimate the state using RGB-D observations, where SAM \cite{carion2025sam} initializes and grounds keypoints that are then tracked in real time by TAPNext++ \cite{jung2026tapnext++}. We then use these keypoint poses to fit and reconstruct the full object state. See App. \ref{app:perception} for more details on the perception pipeline. The experiments are conducted with two KUKA LBR iiwa manipulators with the MPC solver running on an NVIDIA RTX4070 Ti GPU. To calibrate the CP bound, we apply 100 random controls on hardware and compare the predicted and measured one-step state transitions. We then use a rounded empirical error radius of 0.03 m, which covers 87\% of the calibration residuals. In Fig.~\ref{fig:hardware_experiments}\textbf{(a)}, we show that, due to the contact-smoothness and differentiability of our simulator, the manipulators are able to make contact with the rope and navigate it over the obstacle. In Fig.~\ref{fig:hardware_experiments}\textbf{(b)} we visualize the robust tubes generated at selected timestamps which account for both dynamical and measurement uncertainty, showing that we remain safe despite uncertainty. Finally, in Fig.~\ref{fig:hardware_experiments}\textbf{(c)} we show the manipulators successfully completing a cloth folding task, demonstrating that our simulator and solver extend to high-dimensional deformable object tasks.

\vspace{-6pt}
\section{Discussion, Limitations, and Conclusion}
\vspace{-3pt}
This paper presents CORD-SLS, a framework for safe, contact-rich deformable object manipulation that integrates differentiable simulation, robust GPU-accelerated output-feedback control, and data-driven uncertainty calibration. Through smooth contact dynamics, differentiable simulation, and conformal prediction, CORD-SLS enables real-time trajectory optimization for high-dimensional deformable objects through intermittent contact, while maintaining high-probability safety. 
\vspace{-10pt}
\paragraph{Limitations} Contact smoothing introduces errors near contact transitions, which can be handled by \cite{li2026certified}. Deformable-object perception remains challenging under occlusion and complex topologies, which can be addressed through more sophisticated perception pipelines, such as multi-view RGB-D sensing or learned output equations. Lastly, memory usage for higher-dimensional systems and long horizons becomes a bottleneck, motivating the use of low-rank SLS response operators.

\clearpage

\bibliography{contents/references}

\newpage
\appendix

\setcounter{theorem}{0}
\renewcommand{\thetheorem}{\Alph{section}.\arabic{theorem}}

\begin{center}
    \Large \textbf{Appendices}
\end{center}

In the following, we provide an overview of our appendices. In App. \ref{app:proofs}, we provide a proof of Theorem \ref{thm:perceived_to_true_tubes}, which shows how perceived-state conformal tubes imply true-state containment after inflating the tubes by the perception-error bound. In App. \ref{app:perception}, we provide additional details on the hardware perception pipelines used for the rope manipulation and cloth folding experiments, including the camera setup, SAM3-based segmentation, TAPNext++ keypoint tracking, RGB-D reconstruction, spline fitting, and state reconstruction procedures. In App. \ref{app:rl_convergence}, we provide additional reinforcement learning training results comparing PPO and APG on the Lift Rope and Drag Rope tasks, demonstrating the improved sample efficiency enabled by differentiating through our simulator. In App. \ref{app:mpc_tubes}, we provide additional state-feedback MPC tube visualizations for the drag rope and fold cloth tasks, and discuss how replanning affects the relationship between executed trajectories and previously computed robust tubes. In App. \ref{app:output_feedback_tubes}, we provide additional output-feedback tube visualizations for the rope manipulation task, showing rollouts under random dynamical and measurement disturbances when controlled by the robust perception-based controller synthesized by CORD-SLS. In App. \ref{app:simulator}, we compare our differentiable simulator against prior methods in terms of simulation accuracy and speed.

\section{Proofs}\label{app:proofs}

For convenience, we restate Theorem \ref{thm:perceived_to_true_tubes} and then provide the proof.

\begin{theorem}[True-state containment from perceived-state conformal tubes]
\label{thm:perceived_to_true_tubes_app}
Let $y_{k+1}=q_k+e_k$ with $\|e_k\|_2\le\rho$, and define
$\bar q_k:=y_{k+1}$. Suppose conformal calibration and robust SLS synthesis yield
perceived-state and control tubes $\bar{\mathcal{T}}_k^q$ and $\mathcal{T}_k^u$
satisfying
$\mathbb{P}
    \left[
        \bar q_k\in\bar{\mathcal{T}}_k^q,\;
        u_k\in\mathcal{T}_k^u,\;
        k=0,\ldots,N
    \right]
    \ge 1-\delta$. 
Then the true state and control trajectories satisfy
$\mathbb{P}
    \left[
        q_k\in
        \bar{\mathcal{T}}_k^q\oplus \rho\mathbb{B}^{n_q},\;
        u_k\in\mathcal{T}_k^u,\;
        k=0,\ldots,N
    \right]
    \ge 1-\delta$. 
\end{theorem}

\begin{proof}
On the event that
$\bar q_k\in\bar{\mathcal{T}}_k^q$ and $u_k\in\mathcal{T}_k^u$ for all $k$, we have
\[
    q_k=\bar q_k-e_k .
\]
Since $\|e_k\|_2\le\rho$,
\[
    q_k\in \bar{\mathcal{T}}_k^q\oplus \rho\mathbb{B}^{n_q}
\]
for all $k$. The perceived-state containment event has probability at least
$1-\delta$, so the inflated true-state containment event has the same probability
lower bound.
\end{proof}

\section{Perception}\label{app:perception}

We describe in more detail the hardware experimental setups and perception architecture for the rope manipulation experiment, followed by the cloth folding experiment. For both experiments, we use two Kuka iiwa manipulators with Schunk actuators for gripping. We use two computers, one equipped with a NVIDIA RTX4090, in charge of running the perception pipeline and communicating with the manipulators, and another equipped with an NVIDIA RTX4070 Ti Super dedicated to running the solver. The two computers communicate via ROS2 through an Ethernet connection.

We mount two Orbbec Gemini 336L depth cameras to the workspace, with one camera facing the front of the workspace and the other facing the back. Both cameras and manipulators are calibrated to report measurements in a common workstation frame.

\subsection{Rope}

\begin{figure}[H]
    \centering
    \includegraphics[width=\linewidth]{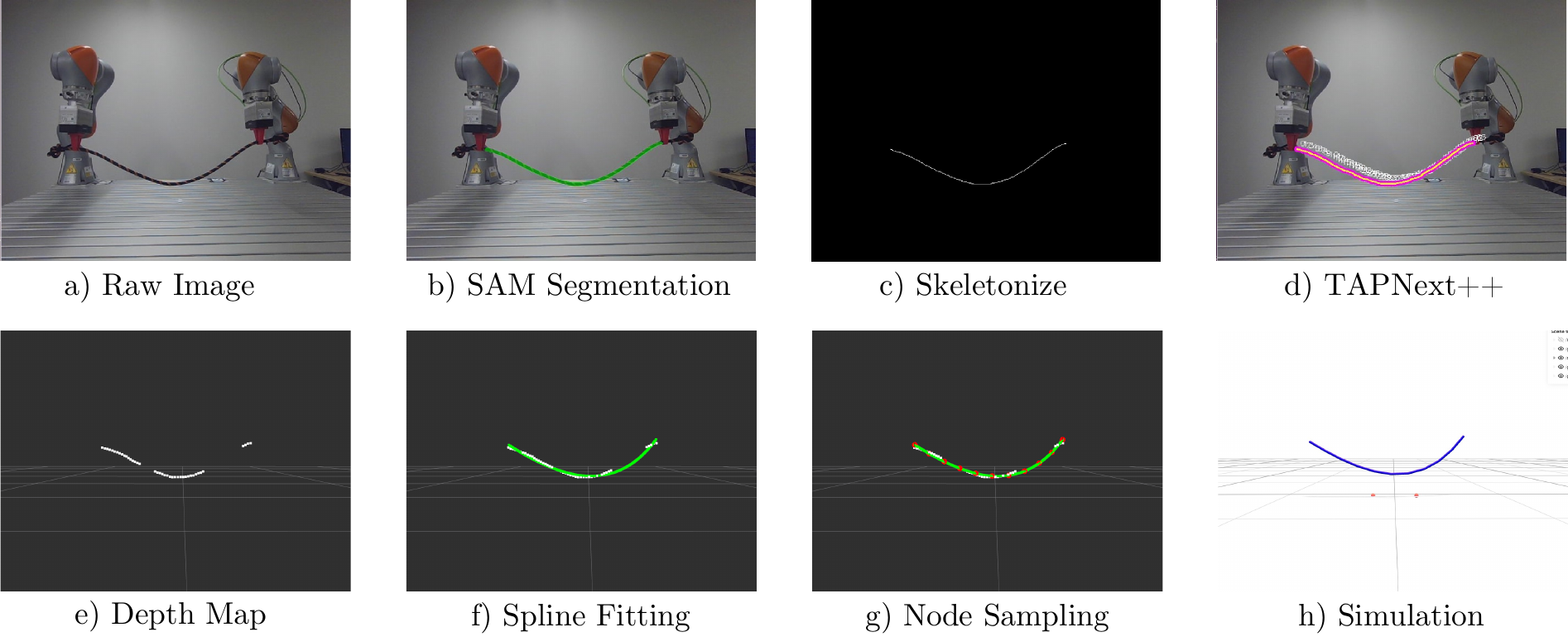}
    \caption{Perception pipeline for rope state estimation. Raw front-camera images (a) are processed by a SAM3 segmentation network (b). The resulting mask is skeletonized (c), and skeleton points are tracked using TAPNext++ (d). RGB points are then associated with their corresponding RGB-D depth map (e). A spline is fit through the recovered 3D points (f), and fixed inter-node distances are sampled along the spline (g) to reconstruct the rope state (h).}
    \label{fig:rope_perception_pipeline}
    \vspace{-15pt}
\end{figure}

We use an approximately 1 cm diameter braided rope with a total length of 1.5 meters. In order to fit the rope within the workspace, we tie both ends of the rope, reducing its effective length to approximately 0.8 m. For our simulation model, we use 10 links.

Our primary perception pipeline combines semantic segmentation with keypoint tracking. In Fig.~\ref{fig:rope_perception_pipeline}, we show snapshots of our perception pipeline. First, we use Segment Anything Model 3 (SAM3) with the prompt ``rope" to segment the rope in the RGB image. The resulting mask is skeletonized to produce a one-pixel-wide centerline through the rope. We then sample a fixed number of points along this centerline and use these points to initialize and periodically re-ground a TapNext++ tracker. Given these RGB pixels, TapNext++ tracks the motion of these keypoints across video frames, enabling continuous tracking of rope keypoints between segmentation updates.

This hybrid design helps balance semantic robustness and runtime efficiency. While SAM3 provides concept-grounded rope detection, we found in our experiments that a single inference takes approximately 200ms, which introduces too much latency for real-time hardware control. TapNext++, on the other hand, runs at lower latency and tracks the rope keypoints frame by frame, but can drift over time as it is not grounded by language or object semantics. We therefore use SAM3 intermittently to reinitialize or reground TapNext++, while TapNext++ provides higher-rate keypoint tracking between SAM3 updates.

The tracked 2D keypoints are then corresponded with the depth image to recover their 3D positions in the workstation frame. To reduce noise and convert these measurements into the state representation used by our simulator, we fit a spline through the 3D points and resample along the spline according to the desired rope segment length. As a result, we obtain our full rope state estimate as an ordered sequence of 3D nodes with fixed segment spacing.

We found that the frequency of SAM3 re-grounding creates a tradeoff between robustness and runtime. Running SAM3 once every second results in TapNext++ running at approximately 15 Hz, while running SAM3 once every 5 seconds allows TapNext++ to run at approximately 25 Hz. This slowdown is likely due to SAM3 competing with TapNext++ for GPU memory and kernel resources. While this pipeline works empirically, the resulting perception rate is still slower than the solver, whose solve time averages less than 20 ms. As a result, the solver must sometimes be delayed to allow the perception to catch up.

To address this latency bottleneck, we also implement a faster rope reconstruction pipeline that trades some perception accuracy for higher control rates. During the initial pickup phase, when the rope is static, we use the full segmentation-and-tracking pipeline to estimate the rope state. After the manipulators grasp the rope, we reconstruct the rope using the known gripper poses and the known rope length between the two manipulators. This reconstruction assumes that the rope shape is dominated by gravitational sag. While this approximation is less accurate than direct perception, it provides a sufficiently reasonable estimate of the rope configuration at substantially higher rates, making it more suitable for closed-loop control.

\subsection{Cloth}

\begin{figure}[H]
    \centering
    \includegraphics[width=\linewidth]{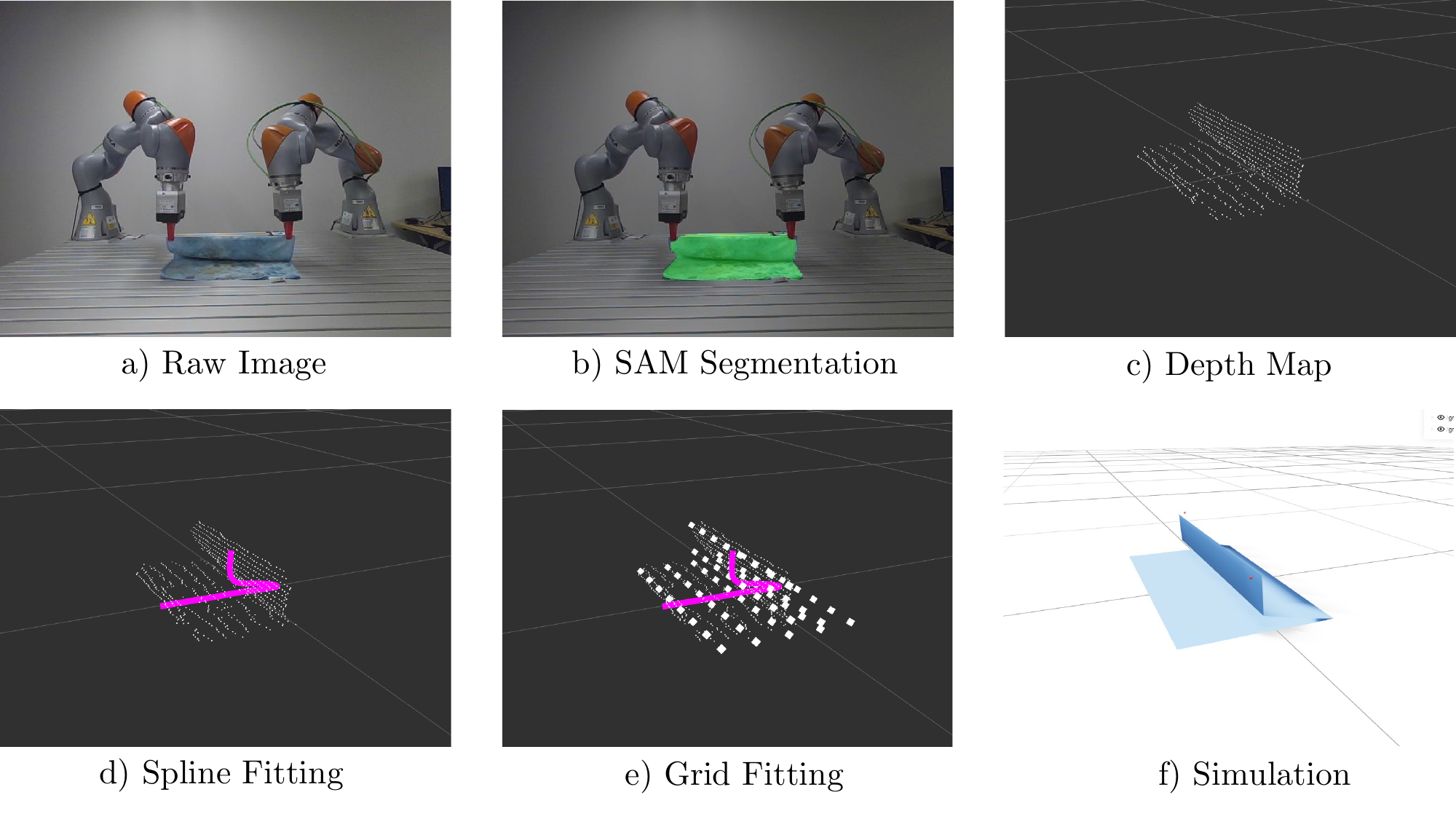}
    \caption{Perception pipeline for cloth state estimate. Raw front-camera images (a) are processed by a SAM3 segmentation network (b). The resulting RGB mask is associated with the depth map to recover a partial 3D point cloud of the cloth (c). A spline is fit to the observed boundary and extended to account for occluded portions, ensuring that the spline's length is consistent with the known cloth length (d). Grid points are then fit to the point cloud and points in occluded regions are generated according to the spline (e). This results in a complete cloth state representation that is provided to the simulator (f).}
    \label{fig:cloth_perception_pipeline}
    \vspace{-15pt}
\end{figure}

We use an approximately 38 cm x 38 cm cloth. For our simulation model, the cloth is discretized into a 7 x 7 grid of cells, corresponding to an 8 x 8 lattice of nodes and yielding 64 cloth nodes in total. The full state has dimension ($64 \times 3 + 3  \times 2 = 198$), consisting of the 3D positions of the cloth nodes and two 3D gripper positions. 

Estimating the cloth state is more challenging than estimating the rope state due to its higher dimensionality, highly deformable, and often self-occluding nature. To make the state estimation for cloth folding more tractable, we simplify perception by assuming that the (x)- and (z)-coordinates are constant along each row of the cloth. In other words, each row is represented as a straight line spanning the cloth width in the (y)-direction, reducing the state estimation problem to recovering the row-wise centerline geometry. To make the physical task consistent with this approach, we constrain the positions of the grippers in the optimization to differ in the (x)- and (z)-plane by at most 5 cm. We summarize and show snapshots of the perception pipeline (Fig.~\ref{fig:cloth_perception_pipeline}).

Our cloth perception pipeline follows the same general structure as the rope perception pipeline. First, we pass the raw images through a SAM3 segmentation pipeline prompted with ``cloth." The resulting segmentation mask is corresponded with the depth map to recover 3D points on the visible cloth surface. However, during folding, portions of the cloth may be occluded by the folded layer and therefore not visible to the front camera and its depth camera. To handle this, we project the recovered points onto the (x)-(z) plane, yielding a 2D point distribution, and fit a centerline spline through these points. Using this spline, we calculate the total arclength and compare this with the expected length of the cloth. If the arclength is significantly shorter than the known cloth length, we treat the missing length as the occluded portion of the cloth.

Lastly, we fit grid points corresponding to the cloth lattice to the visible point cloud and add additional grid points behind the visible cloth to account for occluded regions. This reconstructed state is then passed to the simulator for planning.

\section{RL training convergence}\label{app:rl_convergence}
\begin{figure}[H]
    \centering
    \begin{subfigure}[b]{6.5cm} 
        \centering
        \includegraphics[width=\textwidth]{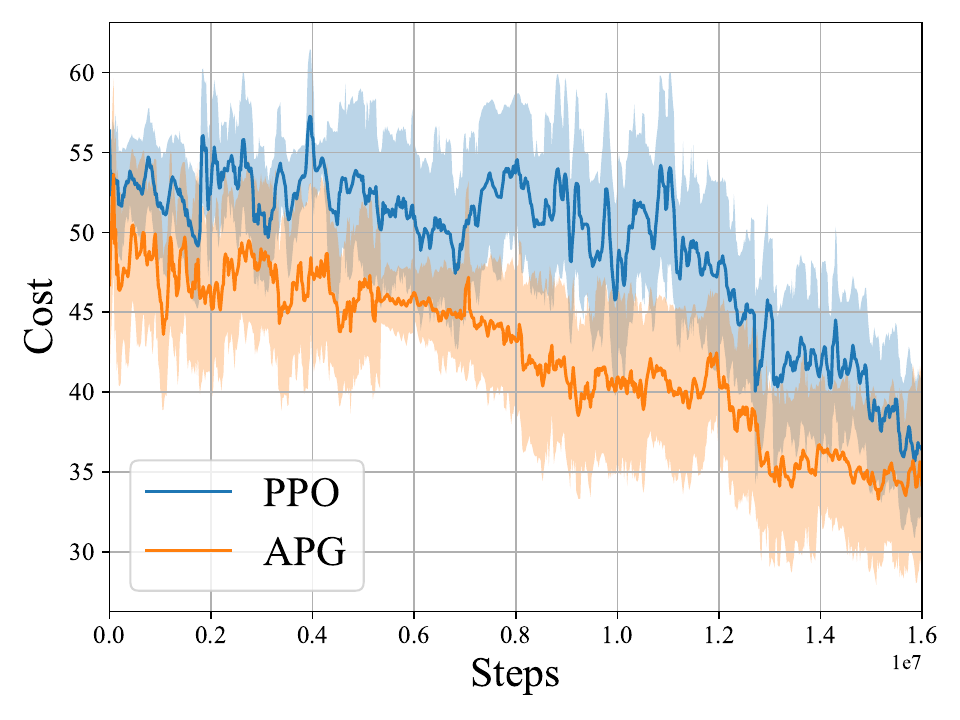}
        \subcaption{}
    \end{subfigure}
    \begin{subfigure}[b]{6.5cm}
        \centering
        \includegraphics[width=\textwidth]{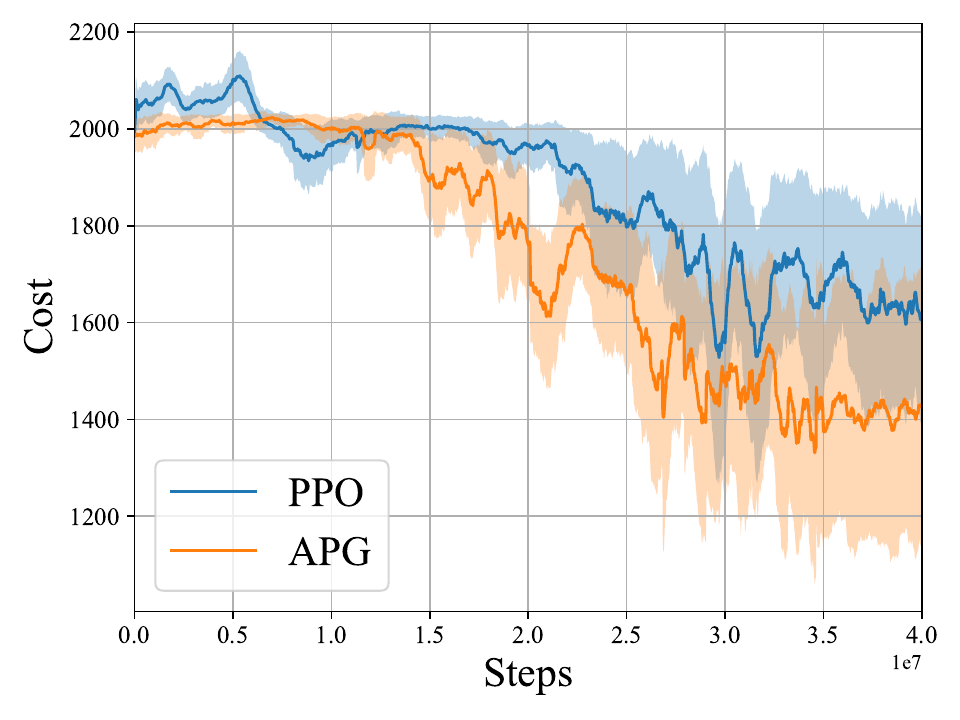}
        \subcaption{}
    \end{subfigure}
    \caption{RL training convergence for the (a) Lift Rope task, and (b) Drag Rope task.}
    \label{fig:rl_convergence}
\end{figure}

In Fig.~\ref{fig:rl_convergence}, we compare the convergence rates of PPO and APG on the Lift Rope and Drag Rope tasks. Because of the differentiability of our simulator, APG can exploit analytical policy gradients obtained by differentiating through the simulator, which accelerates training efficiency. PPO, on the other hand, cannot utilize this gradient information. In both tasks, APG converges to a lower cost faster than PPO. This improvement is due to the contact-smoothed dynamics, which provide lower-variance gradient estimates and improve sample efficiency. Specifically, in the Lift Rope task, APG converges to the same cost as PPO on average in 45.95\% (median 36.01\%) fewer steps than PPO and in 16.63\% (median 27.43\%) fewer steps for the Drag Rope task. These results demonstrate that our differentiable simulator can be used to accelerate RL policy learning.

\section{State Feedback MPC Tubes}\label{app:mpc_tubes}
\begin{figure}[H]
    \centering
    \includegraphics[width=\linewidth]{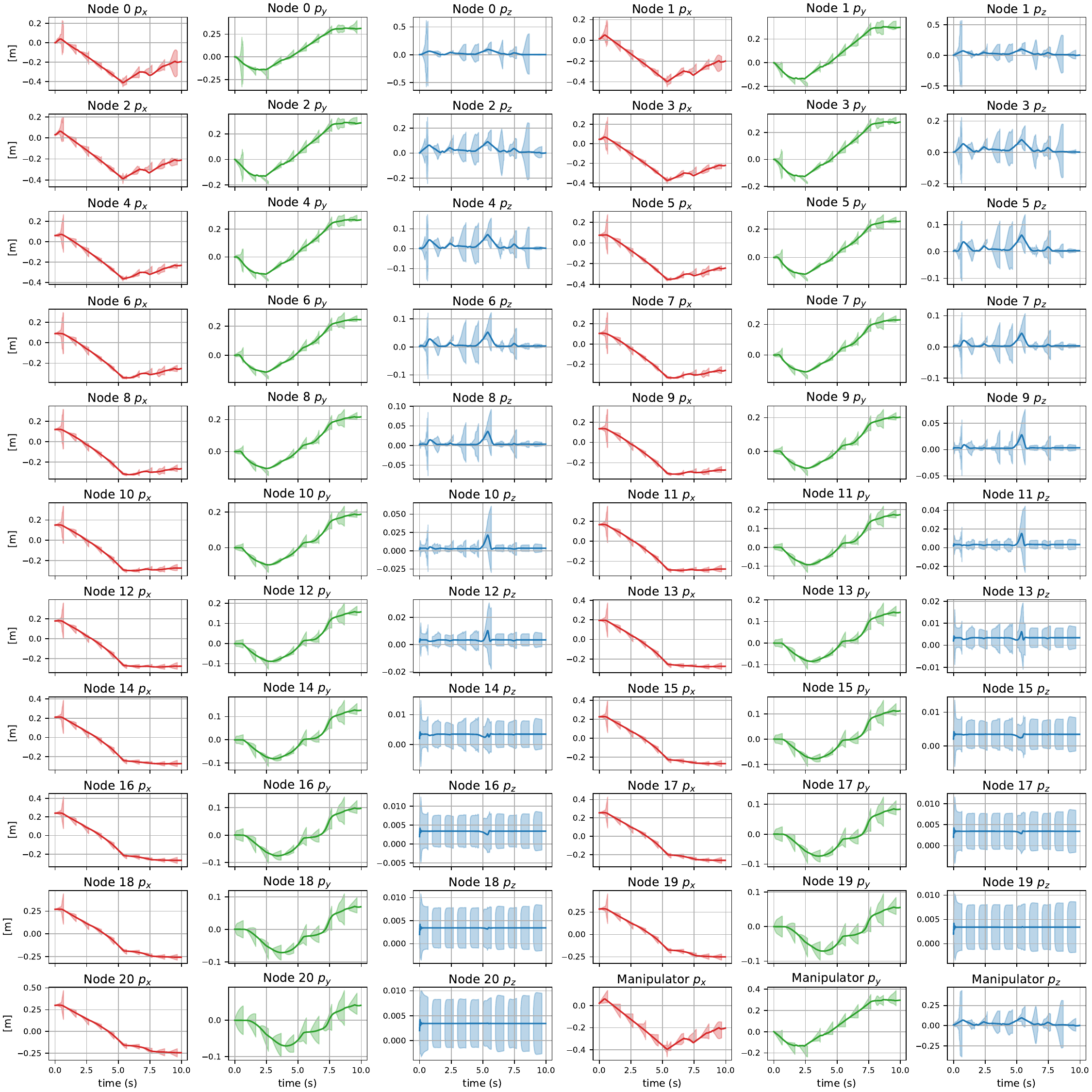}
    \caption{Forward reachable tubes under MPC for the drag rope task for all 66 states.}
    \label{fig:drag-rope-mpc-tubes}
\end{figure}

Fig.~\ref{fig:drag-rope-mpc-tubes} shows the forward reachable tubes of the drag rope task in Fig.~\ref{fig:drag-rope} at selected MPC steps, overlaid with the trajectory (dark line) realized by closed-loop MPC with replanning. We note that in an MPC setting, the executed trajectory may not coincide with the tubes computed at previous time steps because the controller replans after each execution step. Disturbances and the finite prediction horizon can cause subsequent optimizations to select a lower-cost trajectory that deviates from the previously planned path and its associated robust tube. For guaranteed robust recursive feasibility, an approach such as \cite{leeman2025guaranteed} can be applied. Alternatively, rollouts using the robust feedback controller generated by SLS are guaranteed to remain within the computed tubes, but only over the finite horizon for which the controller was synthesized, such as in Fig.~\ref{fig:output_feedback_all_tubes}, which presents full-horizon tubes.

Fig.~\ref{fig:fold-cloth-mpc-tubes} shows the forward reachable tubes for the fold cloth task in Fig.~\ref{fig:fold-cloth} at selected MPC steps. A constraint limits the cloth height to 3 cm in the $z$ direction. Because MPC replans after each execution step, the executed trajectory does not necessarily remain within tubes computed at previous time steps. Nevertheless, by enforcing constraints on the tube-valued predictions during each replan, the controller robustly satisfies the $z$-height constraint throughout the task.

\begin{figure}[H]
    \centering
    \includegraphics[width=\linewidth]{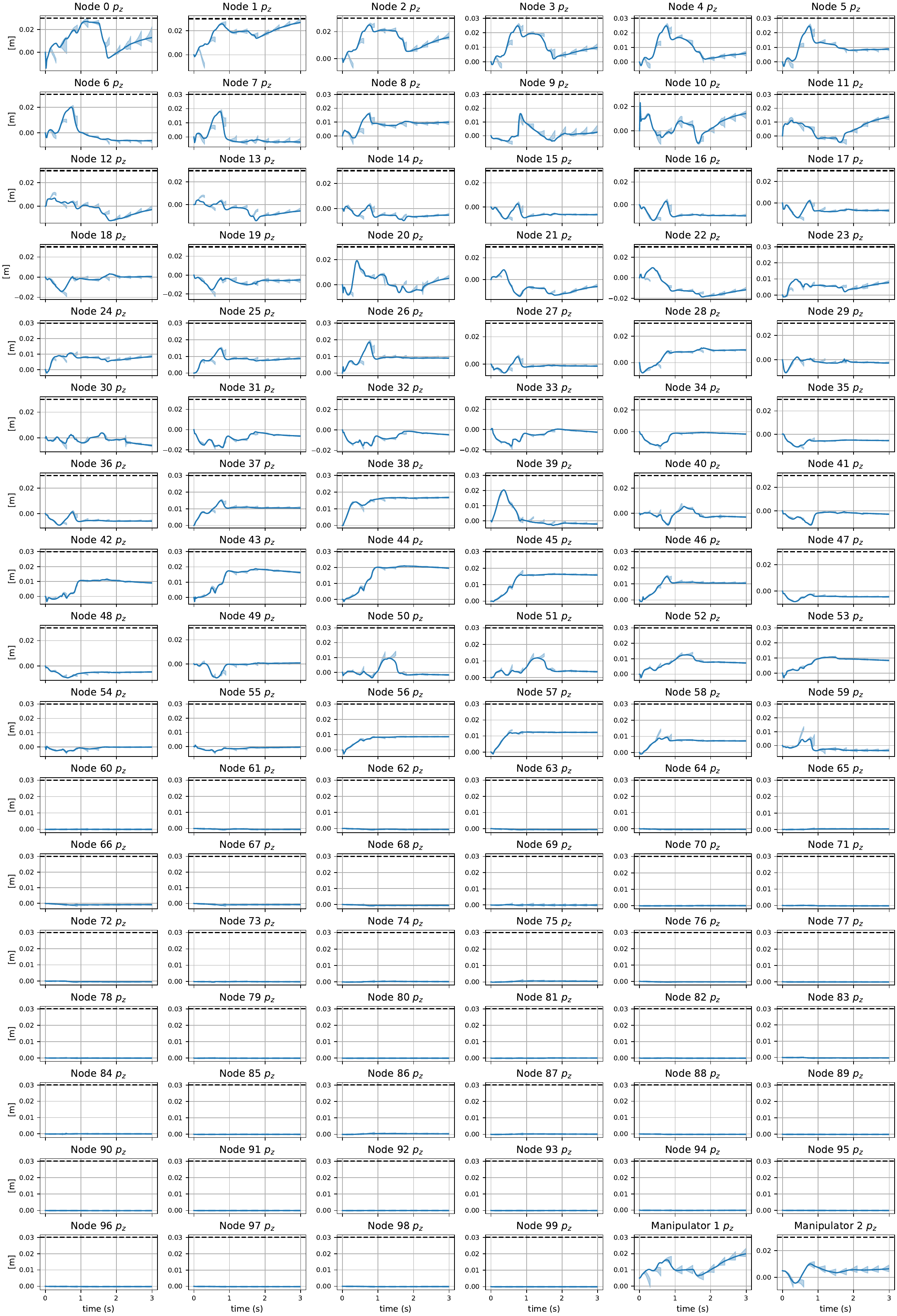}
    \caption{Forward reachable tubes under MPC for the fold cloth task for the $z$-coordinate states.}
    \label{fig:fold-cloth-mpc-tubes}
\end{figure}

\section{Output Feedback Tubes}\label{app:output_feedback_tubes}

We show the tubes for all 39 states for the output feedback experiment in Fig.~\ref{fig:output_feedback}. The rollouts are evaluated under random dynamical and measurement disturbances, while control inputs are generated by the robust observer-based controller synthesized by CORD-SLS.

\begin{figure}[H]
    \centering
    \includegraphics[width=\linewidth]{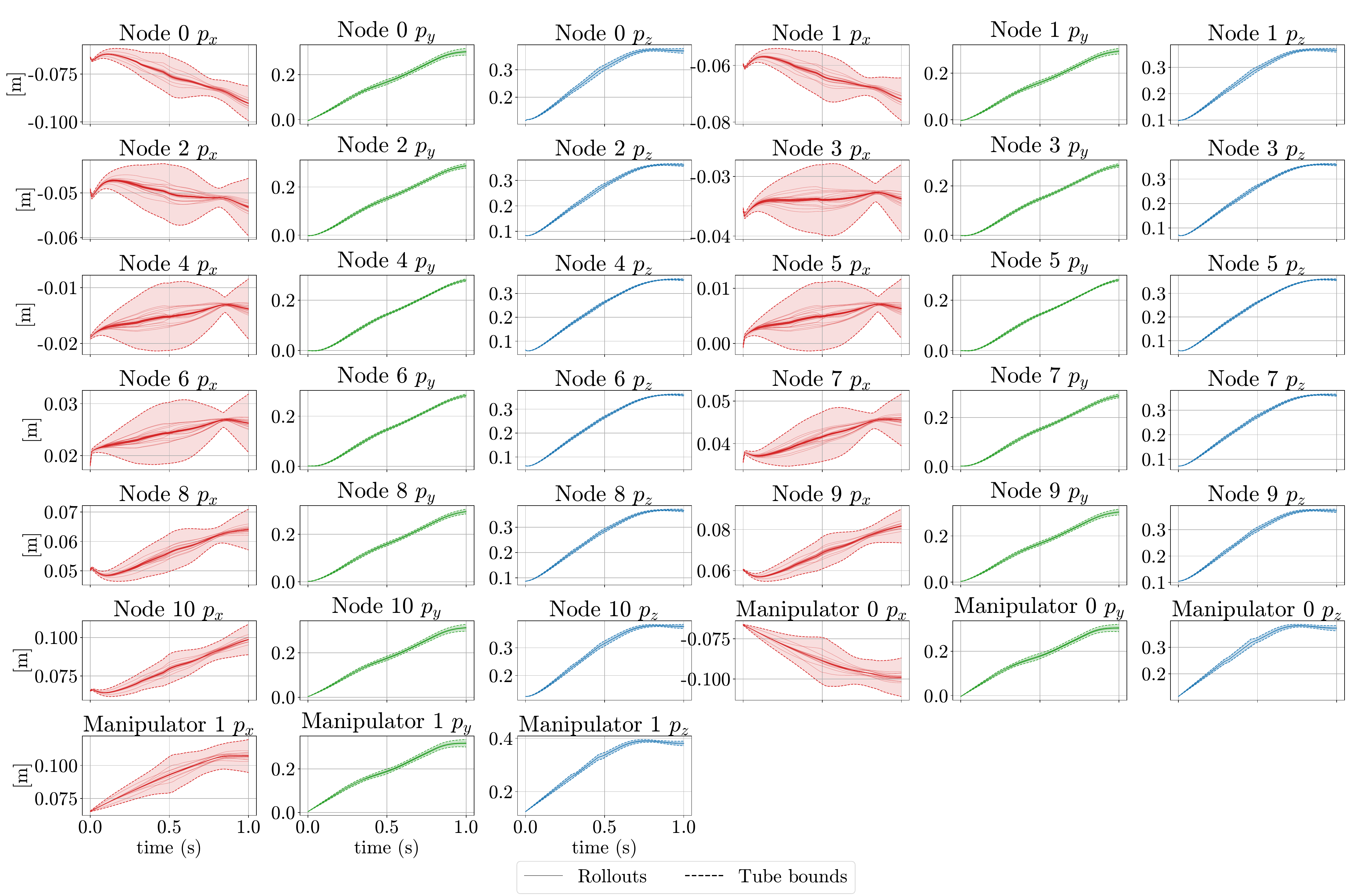}
    \vspace{-15pt}
    \caption{Forward reachable tubes for the output-feedback rope manipulation task for all 39 states. The red trajectories represent the states corresponding to the $x$ dimension, green the $y$ dimension, and blue the $z$ dimension.}
    \label{fig:output_feedback_all_tubes}
\end{figure}

\section{Simulator Comparison}\label{app:simulator}
Table~\ref{tab:sim-speed} reports the runtime of forward simulation and dynamics Jacobian computation for a 20-node rope and a 100-node cloth. Our simulator consistently achieves the fastest execution time among the compared differentiable simulators. Furthermore, matrix prefactorization significantly reduces the overhead of gradient computation, making dynamics Jacobian computation time comparable to forward simulation.

DEFORM \cite{chen2024differentiable} exhibits slower simulation performance because its computation pipeline is not compiled into GPU-native code.
DaXBench \cite{chen2022daxbench} employs different simulation methods for ropes and cloth, using the Material Point Method (MPM) for ropes and a mass-spring model for cloth, which leads to different computational costs.

\begin{table}[H]
\centering
\begin{tabular}{llcc}
\toprule
\multicolumn{1}{c}{}   &                                      & Forward simulation [ms] & Dynamics Jacobian computation [ms] \\
\midrule
\multirow{3}{*}{Rope}  & Ours                                 & 1.095 $\pm$ 0.162       &   5.198 $\pm$ 2.533                \\ 
                       & DEFORM \cite{chen2024differentiable} & 49.03 $\pm$ 20.04       & 2128.81 $\pm$ 55.14                \\
                       & DaXBench \cite{chen2022daxbench}     & 786.48 $\pm$ 76.23      & 11561.57 $\pm$ 122.07              \\
\midrule
\multirow{2}{*}{Cloth} & Ours                                 & 15.08 $\pm$ 1.209       &   10.06 $\pm$ 0.983                \\
                       & DaXBench \cite{chen2022daxbench}     & 21.71 $\pm$ 0.722       &  110.19 $\pm$ 2.663                \\
\bottomrule
\end{tabular}
\caption{Comparison of simulation time across different differentiable simulators.}
\label{tab:sim-speed}
\end{table}

We next benchmark the simulation accuracy of our method against DEFORM. To construct the benchmark dataset, we collect 100 samples of randomized manipulation actions on hardware while running our perception pipeline to estimate the rope state. For each sample, we record the current state, control action, and resulting next state, i.e. ($x^+ = f(x,u)$). We note that due to the latency in the perception pipeline, we are unable to directly observe the instantaneous post-action state and therefore cannot reliably capture transient swinging or momentum-driven dynamics. Instead, the recorded next state corresponds to the settled configuration of the rope after the motion has dissipated, making the dataset more representative of quasi-static behavior.

We evaluate both simulators under two settings. First, we report the one-step prediction accuracy of the base simulator, measured as the mean L2 distance between the predicted and observed next states. Second, we evaluate each simulator with an augmented learned residual model. The original DEFORM framework employs a learned residual consisting of a network architecture to compensate for modeling errors. To ensure a fair comparison, we train an identical residual architecture for both simulators using the same training procedure and hyperparameters. The dataset is split into 70\% training and 30\% validation data, and we report the mean validation error for the residual-corrected models. For all experiments, we set the rope length to 0.8m and discretize it into 10 links.

A key difference between our method and the DEFORM simulator is that DEFORM is formulated as a fully dynamic simulator, whereas our method assumes quasi-static behavior. Consequently, directly comparing a single forward simulation step would be unfair, since the recorded hardware data corresponds to settled equilibrium configurations rather than transient dynamics. Therefore, to account for this discrepancy when evaluating DEFORM, after the commanded action, we repeatedly propagate the dynamics with zero control input until the rope converges to an equilibrium configuration. This produces predictions that are consistent with the quasi-static states observed by our perception system.

Table~\ref{tab:sim-accuracy} summarizes the resulting one-step prediction accuracy. DEFORM achieves approximately 7\% lower prediction error than our simulator. This performance gap is expected, as DEFORM models additional physical degrees of freedom, including twisting deformation, through the discrete elastic rods formulation, whereas our simulator uses an elastic energy model that captures only stretching and bending deformation. As a result, certain higher-order deformation effects are not explicitly represented, leading to a loss in model fidelity. However, we note that the computational efficiency of our simulator and planner enables rapid replanning during execution, allowing the system to quickly compensate for modeling inaccuracies and disturbances. Finally, we show that both simulators benefit from the learned residual correction model, with DEFORM maintaining a slight advantage due to its stronger underlying physical model.

\begin{table}[H]
\centering
\begin{tabular}{llc}
\toprule
\multicolumn{1}{c}{}   &                                            & Mean L2 Error [m] \\
\midrule
\multirow{3}{*}{Rope}  & Ours                                               & 0.0150 $\pm$ 0.0040     \\  
                       & DEFORM \cite{chen2024differentiable}               & 0.0139 $\pm$ 0.0029     \\
                       & Ours with learned Residual                         & 0.0075 $\pm$ 0.0068     \\
                       & DEFORM with learned Residual \cite{chen2024differentiable} &  0.0071 $\pm$ 0.0032 \\
\bottomrule
\end{tabular}
\caption{Comparison of simulation accuracy of our method against DEFORM.}
\label{tab:sim-accuracy}
\end{table}

\end{document}